\documentclass[runningheads]{llncs}
\usepackage[T1]{fontenc}
\usepackage{graphicx}
\def\techreport{}

\usepackage{amsmath}

\usepackage{amssymb}
\usepackage{multirow}
\usepackage{graphicx}
\usepackage{subcaption}
\usepackage{dsfont}

\usepackage[colorinlistoftodos]{todonotes}

\usepackage[utf8]{inputenc}

\newcommand\demo{}

\usepackage{mathtools}
\usepackage{tikz}
\usetikzlibrary{automata,positioning,angles,quotes}

\usepackage{wrapfig}

\usepackage{booktabs}  % For professional looking tables
\usepackage{array}     % For custom column alignment

\newcommand{\citep}[1]{{\cite{#1}}}

\definecolor{color1}{rgb}{0.1,0.498039215686275,0.9549019607843137}
\definecolor{alizarin}{rgb}{0.82, 0.1, 0.26}
\definecolor{antiquewhite}{rgb}{0.98, 0.92, 0.84}
\definecolor{azure}{rgb}{0.94, 1.0, 1.0}
\definecolor{offwhite}{rgb}{0.98, 0.97, 0.97}
\definecolor{pigment}{rgb}{0.2, 0.2, 0.6}

\usepackage[utf8]{inputenc} % allow utf-8 input
\usepackage[T1]{fontenc}    % use 8-bit T1 fonts
\usepackage{url}            % simple URL typesetting
\usepackage{booktabs}       % professional-quality tables
\usepackage{amsfonts}       % blackboard math symbols
\usepackage{nicefrac}       % compact symbols for 1/2, etc.
\usepackage{microtype}      % microtypography
\usepackage{stackengine}

\usepackage{tikz,pgfplots}

\usepgfplotslibrary{fillbetween}
\usepgfplotslibrary{colormaps}
\usepgfplotslibrary{patchplots}

\newcommand{\cons}[0]{C_{\theta_{1:N}}}
\newcommand{\consh}[0]{\hat{C}_{\theta_{1:N}}}
\renewcommand{\xRightarrow}[2][]{\ext@arrow 0359\Rightarrowfill@{#1}{#2}}

\usepackage{xcolor}
\definecolor{color1}{rgb}{0.1,0.498039215686275,0.9549019607843137}
\definecolor{alizarin}{rgb}{0.82, 0.1, 0.26}
\definecolor{antiquewhite}{rgb}{0.98, 0.92, 0.84}
\definecolor{azure}{rgb}{0.94, 1.0, 1.0}
\definecolor{offwhite}{rgb}{0.98, 0.97, 0.97}
\definecolor{pigment}{rgb}{0.2, 0.2, 0.6}

\usepackage{bm}
\newcommand{\upmdp}[0]{\mathcal{M}^{\mathbb{P}}_\Theta}
\DeclareMathOperator\supp{supp}

\providecommand*{\cupdot}{%
  \mathbin{%
    \mathpalette\@cupdot{}%
  }%
}
\newcommand*{\@cupdot}[2]{%
  \ooalign{%
    $\m@th#1\cup$\cr
    \sbox0{$#1\cup$}%
    \dimen@=\ht0 %
    \sbox0{$\m@th#1\cdot$}%
    \advance\dimen@ by -\ht0 %
    \dimen@=.5\dimen@
    \hidewidth\raise\dimen@\box0\hidewidth
  }%
}

\newcommand{\until}{\, \mathcal{U} \,}

\newenvironment{proofsketch}{%
  \proof}{\endproof}
  
\usepackage{colortbl}

\makeatletter
\RequirePackage[bookmarks,unicode,colorlinks=true]{hyperref}%
   \def\@citecolor{blue}%
   \def\@urlcolor{blue}%
   \def\@linkcolor{blue}%

\def\orcidID#1{\smash{\href{http://orcid.org/#1}{\protect\raisebox{-1.25pt}{\protect\includegraphics{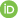}}}}}
\makeatother

\usepackage[framemethod=TikZ]{mdframed}

\newmdtheoremenv[
  innerleftmargin=5pt,
  innerrightmargin=5pt,
  innertopmargin=5pt,
  innerbottommargin=5pt,
  skipabove=10pt,
  skipbelow=10pt
]{myproblem}{Problem}

\newcommand{\startpara}[1]{{%
\vskip6pt\noindent
{\bf #1.}}}

\usepackage{bbding}

\usepackage{color}

\begin{document}
\title{Certifiably Robust Policies \\ for Uncertain Parametric Environments}
 \author{Yannik Schnitzer~\orcidID{0000-0001-7406-3440} \and Alessandro Abate~\orcidID{0000-0002-5627-9093} \and David Parker~\orcidID{0000-0003-4137-8862}}
%\author{Anonymous submission}

%
\authorrunning{Y. Schnitzer et al.}
\institute{University of Oxford, Oxford, UK \email{\{yannik.schnitzer,alessandro.abate,david.parker\}@cs.ox.ac.uk}}
 
%\institute{}
%
\maketitle              % typeset the header of the contribution
\begin{abstract}

We present a data-driven approach for producing policies that are provably robust across unknown stochastic environments. Existing approaches can learn models of a single environment as an interval Markov decision processes (IMDP) and produce a robust policy with a probably approximately correct (PAC) guarantee on its performance. However these are unable to reason about the impact of environmental parameters underlying the uncertainty. We propose a framework based on parametric Markov decision processes with unknown distributions over parameters. We learn and analyse IMDPs for a set of unknown sample environments induced by parameters.
The key challenge is then to produce meaningful performance guarantees that combine the two layers of uncertainty: (1) multiple environments induced by parameters with an unknown distribution; (2) unknown induced environments which are approximated by IMDPs. We present a novel approach based on scenario optimisation that yields a single PAC guarantee quantifying the risk level for which a specified performance level can be assured in unseen environments, plus a means to trade-off risk and performance.
We implement and evaluate our framework using multiple robust policy generation methods on a range of benchmarks. We show that our approach produces tight bounds on a policy's performance with high confidence.

%\ys{check for clearity}
% \keywords{Markov decision processes  \and Uncertainty \and Verification.}
\end{abstract}
\section{Introduction}
\label{sec:intro}
% Intuitively, speaking in the drone delivery example from Figure~\ref{fig:drone}, this sampling of partially unknown MDPs can be interpreted as taking the drone outside on a new day, encountering a new instantiation of the upMDP with new parameters for, e.g., wind, temperature, or precipitation. It may be the case that for some parameters in our model, we can measure the respective quantities in our current environment and are aware of the exact dependency of the stochastic behavior on these quantities. However, it may as well be the case that we can not exactly infer the respective probabilities in our system since (a) the quantities determining a parameter in our model are not exactly measurable or observable in our environment; (b) the exact dependency of a parameter on observable quantities is unknown. 
% \begin{wrapfigure}[16]{r}{0.5\textwidth}
%     \centering
%     \input{text/procedure_flow}
%     \caption{Procedure}
%     \label{fig:procedure}
% \end{wrapfigure}+

Ensuring the safety and robustness of autonomous systems in safety-critical tasks, such as unmanned aerial vehicles (UAVs), %~\citep{DBLP:journals/sttt/BadingsCJJKT22, uav},
robotics %~\citep{sharedautonomys-2022/7, riskawaremotion-2022/2, decisionmakingu-2022/}
or autonomous control, %~\cite{DBLP:journals/corr/abs-2312-06344,DBLP:conf/qest/RickardBRA23},
is paramount. A standard model for sequential decision making in these settings is a Markov decision process (MDP), which provides a stochastic model of the environment.
However, real-world dynamics are complex, not fully known, and may evolve over time. 
Reasoning about \textit{epistemic uncertainty}, which quantifies the lack of knowledge about the environment,
can help construct \textit{robust} policies that perform well across multiple possible stochastic environments. 

% Therefore, models with \textit{epistemic uncertainty}~\citep{badings2023decisionmaking,DBLP:journals/sttt/BadingsCJJKT22,DBLP:journals/corr/charpentier2022,coppola2024datadriven,ghosh2021generalization,DBLP:conf/tacas/Junges0DTK16,DBLP:journals/ior/NilimG05,DBLP:conf/qest/RickardBRA23,DBLP:journals/mor/WiesemannKR13} are considered which account for uncertainty resulting from a lack of knowledge of the system or the environment under investigation. Uncertain models help construct \textit{robust} policies that perform well across the various possible stochastic environments. %~\citep{DBLP:conf/aaai/BadingsA00PS22, DBLP:conf/aaai/BadingsRA023, DBLP:journals/corr/abs-2404-01726}.

\begin{figure}[h]
    \centering
    \begin{subfigure}[b]{0.47\textwidth}
        \centering
        \hspace{-15pt}
        \scalebox{0.55}{
        \input{text/UAV-plan}
        }
        \caption{UAV motion planning environment with sample trajectories~\citep{DBLP:journals/sttt/BadingsCJJKT22}.}
        \label{fig:drone}
    \end{subfigure}
    \hfill
    \begin{subfigure}[b]{0.47\textwidth}
        \centering
        \includegraphics[width = \textwidth]{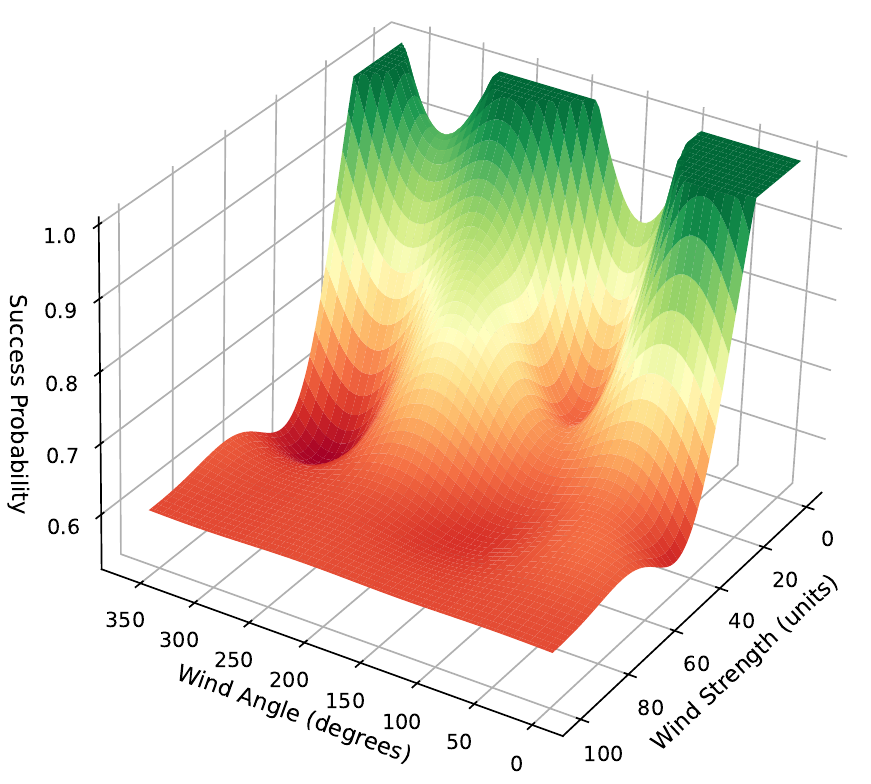} % Placeholder for the second subfigure
        \caption{Probability of task completion $J(\pi,\theta)$ for different parameter valuations $\theta$.}
        \label{fig:droneeval}
    \end{subfigure}
    \caption{Example parametric environment with induced performance function.}
    \vspace{-16pt}
\end{figure}

Consider the UAV motion planning problem shown in Figure~\ref{fig:drone}, based on~\cite{DBLP:journals/sttt/BadingsCJJKT22}. The goal is to navigate the drone safely to the target zone (green box) whilst avoiding obstacles (red regions). The drone's dynamics are influenced by weather conditions, such as wind strength or direction, potentially perturbing the drone from its intended route. These conditions can vary over time and may be difficult to observe exactly.  % and temperature.
In low disturbance conditions (e.g., light wind), the drone can safely take the shorter route to the target (black dashed line). However, the drone should fly safely under all conditions, even if it flies overly cautiously in some. Therefore, a robust policy might take a detour through a less cluttered region of the environment (longer blue dashed line), ensuring a high probability of task completion even under more severe disturbances.

% % \dave{Clarify unobserved params vs. unknown influence?}
% The (stochastic) dynamics of the drone are impacted by environmental parameters such as wind speed or direction. The values of these parameters are unknown - they can vary over time and may be difficult to observe - but some conditions may be more probable than others.

Epistemic uncertainty about the environment can be captured using \emph{uncertain} MDPs, such as \emph{interval MDPs} (IMDPs), which define a range of possible values for the probability of each transition between states of the model~\cite{DBLP:journals/ai/GivanLD00,DBLP:journals/mor/WiesemannKR13}.
Under assumptions of independence, techniques such as robust dynamic programming~\cite{DBLP:journals/mor/Iyengar05,DBLP:journals/ior/NilimG05} can then be used to efficiently generate \emph{robust} policies for these IMDPs, i.e., policies that are optimal under \emph{worst-case} assumptions about the true values of the  transition probabilities.
Furthermore, data-driven approaches, for example based on sampled trajectories through the environment, can be used to simultaneously learn both an IMDP and a robust policy for it~\cite{DBLP:conf/cav/AshokKW19,DBLP:journals/corr/meggendorfer24,DBLP:conf/icml/StrehlL05,DBLP:conf/nips/SuilenS0022},
along with a probably approximately correct (PAC) guarantee on its performance.
% (for a specified confidence level) 

In this paper, we present a general framework for synthesising provably robust policies
in settings where environmental uncertainty is influenced by one or more \emph{parameters},
e.g., wind strength/direction in the UAV example above.
We model environments as \textit{uncertain parametric MDPs} (upMDPs)~\citep{DBLP:journals/sttt/BadingsCJJKT22},
% \cite{DBLP:conf/aaai/BadingsRA023, DBLP:conf/valuetools/Scheftelowitsch17},
comprising a \emph{parameter space} $\Theta$ of which each \emph{parameter valuation} $\theta\in\Theta$ induces a standard MDP.
Furthermore, an (unknown) distribution $\mathbb{P}$ over $\Theta$ represents the likelihood of each parameter valuation.
Based on trajectories through multiple sampled instances of the environment, % (sampled via $\mathbb{P}$), 
our goal is to produce \emph{certifiably robust} policies. 
Instead of making worst-case assumptions across all possible parameter values,
which can be overly conservative, we adopt a risk-based approach,
providing a guaranteed level of performance for the policy with an associated \emph{risk level}
that quantifies the possibility of this level being violated.
We also provide a tuning mechanism to adapt the trade-off between performance and risk.

To quantify the performance of a policy $\pi$ in an MDP induced by parameter valuation $\theta \in \Theta$, we use an \emph{evaluation function} $J(\pi, \theta)$.
Typical examples include the probability to satisfy a specification expressed in temporal logics such as LTL~\cite{DBLP:conf/focs/Pnueli77} or PCTL~\cite{DBLP:journals/fac/HanssonJ94} or an expected reward (see Section~\ref{sec:prelim}). For a fixed policy, $J$ becomes a function in the valuations $\theta$, as depicted in Figure~\ref{fig:droneeval} for the UAV example. When additionally considering the distribution $\mathbb{P}$ over parameter valuations, $J$ becomes a random variable with respect to $\mathbb{P}$ describing the performance likelihood under policy $\pi$ (see Figure~\ref{fig:inducedVal}).
Whereas the dashed vertical lines in Figure~\ref{fig:inducedVal} indicate the worst-case performance of each policy, Figure~\ref{fig:violationrisk} illustrates the risk measure $r(\pi, \tilde{J})$ that we use in this paper: the probability with which performance falls below a specified threshold $\tilde{J}$.

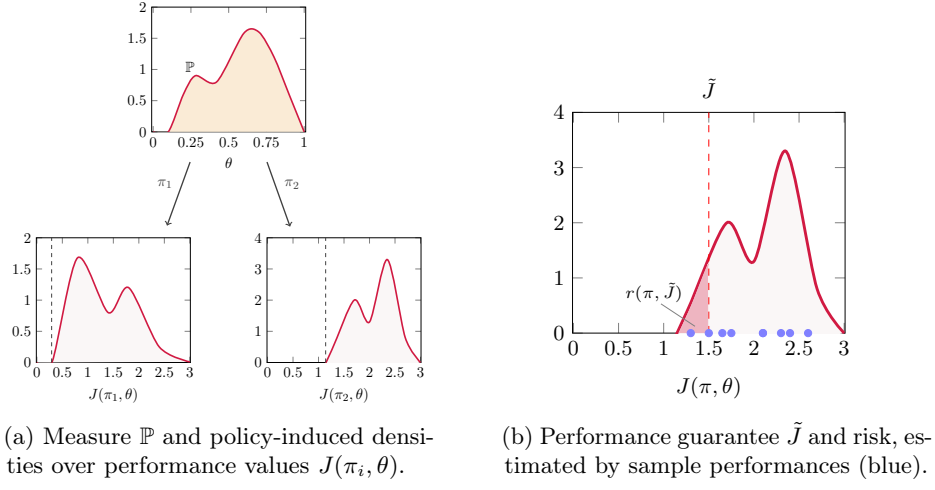
\begin{figure}[t]
    \centering
    \begin{subfigure}[b]{0.46\textwidth}
        \centering
        \includegraphics[width=\textwidth]{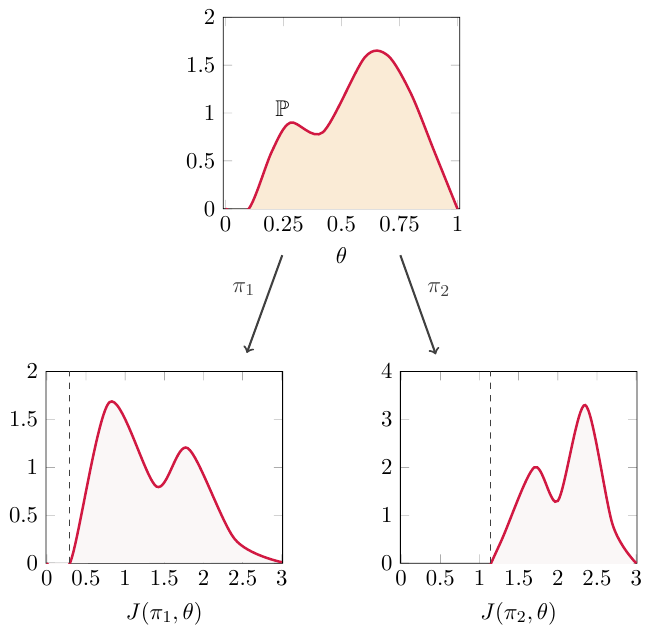}
        \caption{Measure $\mathbb{P}$ and policy-induced densities over performance values $J(\pi_i,\theta)$.}
        \label{fig:inducedVal}
    \end{subfigure}
    \hfill
    \begin{subfigure}[b]{0.46\textwidth}
        \centering
        \hspace{-0.5cm} % Adjust this value to move the figure to the left
            \centering
    \scalebox{0.9}{
    % \begin{tikzpicture}[baseline, remember picture]
    % \begin{axis}[width=2.2in,
    %              height=1.9in,
    %              ymin=0,
    %              ymax=4,
    %              xmin=-0.01,
    %              xmax=3.01,
    %              xlabel={$V(\pi, \theta)$},
    %              ylabel={$\Pr$},
    %              xtick={0, 0.5, 1, 1.5, 2, 2.5, 3},
    %              ylabel style={rotate=-90, xshift=10pt}, % Rotate and move closer to the axis
    %              smooth,
    %              samples=200,
    %              clip=true]
    % \addplot[mark=none, unbounded coords=jump, color=alizarin, line width = 1.2pt, fill = offwhite, name path=axis] coordinates {
    %                 (0,-1)
    %                 (0.3,-1)
    %                 (0.6,-1)
    %                 (0.84,-1)
    %                 (1.26,0.4)
    %                 (1.7,2)
    %                 (2.0, 1.3)
    %                 (2.35,3.3)
    %                 (2.7,0.8)
    %                 (3.0,0)
    %             };
    
    % % Vertical line
    % \addplot[name path=vertline, dashed, red] coordinates {(1.5,0) (1.5,4)};
    
    %     \addplot[mark=none, unbounded coords=jump, color=alizarin!30!offwhite, line width = 0.1pt, name path=axis, fill = alizarin!30!offwhite] coordinates {
    %                 (1.17,0.001)
    %                 (1.3, 0.46)
    %                 (1.48,1.2)
    %                 (1.485,1.1)
    %                 (1.4881,0.001)
    %             };
                
    % \end{axis}
    % \node[] at (2.1,3.6) {\small$\tilde{V}$};
    
    % \node[] (upper) at (-1,0.5) {};
    % \end{tikzpicture}
    % \hspace{20pt}
    \begin{tikzpicture}[baseline, remember picture, ]
    \begin{axis}[width=2.2in,
                 height=1.9in,
                 ymin=0,
                 ymax=4,
                 xmin=-0.01,
                 xmax=3.01,
                 xlabel={$J(\pi, \theta)$},
                 ylabel={},
                 xtick={0, 0.5, 1, 1.5, 2, 2.5, 3},
                 ylabel style={rotate=-90, xshift=10pt}, % Rotate and move closer to the axis
                 smooth,
                 samples=200,
                 clip=true]
    \addplot[mark=none, unbounded coords=jump, color=alizarin, line width = 1.2pt, name path=axis, fill = offwhite] coordinates {
                    (0,-1)
                    (0.3,-1)
                    (0.6,-1)
                    (0.84,-1)
                    (1.26,0.4)
                    (1.7,2)
                    (2.0, 1.3)
                    (2.35,3.3)
                    (2.7,0.8)
                    (3.0,0)
                };
    
    % Vertical line
    \addplot[name path=vertline, dashed, red] coordinates {(1.5,0) (1.5,4)};

    \addplot[mark=none, unbounded coords=jump, color=alizarin!30!offwhite, line width = 0.1pt, name path=axis, fill = alizarin!30!offwhite] coordinates {
                    (1.17,0.001)
                    (1.3, 0.46)
                    (1.48,1.2)
                    (1.485,1.1)
                    (1.4881,0.001)
                };

    % % Blue dots
    \addplot[only marks, mark=*, mark size=1.5pt, blue!50] coordinates {(1.3,0) (1.5,0) (1.65,0) (1.75,0) (2.3,0) (2.1,0) (2.1,0) (2.4,0) (2.6,0)};

    % Blue bars
    % \addplot[only marks, mark=none, forget plot] coordinates {(1.3,0) (1.6,0) (1.7,0) (2.3,0) (2.1,0) (2.1,0) (2.4,0) (2.6,0)};
% \addplot[mark=none, blue!50, line width=1.2pt, forget plot] coordinates {(1.3,0) (1.3,0.2)};
% \addplot[mark=none, blue!50, line width=1.2pt, forget plot] coordinates {(1.6,0) (1.6,0.2)};
% \addplot[mark=none, blue!50, line width=1.2pt, forget plot] coordinates {(1.7,0) (1.7,0.2)};
% \addplot[mark=none, blue!50, line width=1.2pt, forget plot] coordinates {(2.3,0) (2.3,0.2)};
% \addplot[mark=none, blue!50, line width=1.2pt, forget plot] coordinates {(2.1,0) (2.1,0.2)};
% \addplot[mark=none, blue!50, line width=1.2pt, forget plot] coordinates {(2.1,0) (2.1,0.2)};
% \addplot[mark=none, blue!50, line width=1.2pt, forget plot] coordinates {(2.4,0) (2.4,0.2)};
% \addplot[mark=none, blue!50, line width=1.2pt, forget plot] coordinates {(2.6,0) (2.6,0.2)};

    \end{axis}
    \node[] at (2.0,3.6) {\small$\tilde{J}$};
    \node[]at (1.2,0.6) {\scriptsize$r(\pi, \tilde{J})$};
    \draw[gray, thin] (1.3,0.43)  -- (1.8,0.1);
    \end{tikzpicture}
    }
        \caption{Performance guarantee $\tilde{J}$ and risk, estimated by sample performances (blue).} %$r(\pi, \tilde{J})$.}
        \label{fig:violationrisk}
    \end{subfigure}
    \caption{For a fixed policy $\pi$, $J(\pi,\theta)$ is a random variable over performance values with measure $\mathbb{P}$ over valuations $\theta \in \Theta$ (left). We sample performances to bound the risk $r(\pi, \tilde{J})$, i.e., the probability for $J$ to take a value less than $\tilde{J}$ (right).}
    \vspace{-8pt}
\end{figure}
%\ys{}

Deriving policies that are robust, i.e., which achieve high performance across either many or all possible environments, is a challenging problem.
When $\Theta$ is finite, and assuming worst-case performance (i.e., ignoring $\mathbb{P}$), the model is referred to as a \emph{multi-environment MDP}, for which finding an optimal robust memoryless policy is NP-hard, even for just two fully known environments~\citep{DBLP:conf/fsttcs/RaskinS14}.
For general upMDPs, recent work~\cite{DBLP:journals/corr/Rickard23} finds robust policies but assumes that $\mathbb{P}$ can be sampled directly and that the resulting MDPs are fully known.
In our setting, transition probabilities are unknown and are inferred from trajectories.
This is comparable to the aforementioned work on PAC-learning of IMDPs~\cite{DBLP:conf/cav/AshokKW19,DBLP:journals/corr/meggendorfer24,DBLP:conf/icml/StrehlL05,DBLP:conf/nips/SuilenS0022}, 
but these methods assume a single, fixed (but unknown) MDP.
% The generality of our setup thus goes beyond existing verification approaches to upMDPs~\cite{DBLP:journals/sttt/BadingsCJJKT22,DBLP:journals/corr/Rickard23}, which require the sampled environments to be fully known.

In our work, there are \emph{two layers} of uncertainty, resulting from (1) unknown parameter valuations inducing unknown MDPs, sampled from (2) the unknown parameter distribution.
% While exact optimal policy synthesis is out of reach for unknown distributions over infinite parameter spaces,
In this setting, various learning-based methods known as \emph{robust meta reinforcement learning}
have been proposed~\cite{DBLP:conf/nips/CollinsMS20,DBLP:conf/nips/GreenbergMCM23,DBLP:conf/nips/TehBCQKHHP17}.
% aim to learn a robust policy from trajectories of sample environments that generalises to the entire parameter space
% The input to these procedures is a finite set of environments sampled from the distribution $\mathbb{P}$, which themselves are not known.
% The algorithms are only provided with access to trajectories from each environment.
% Robust meta RL comes with different objectives ranging from the most safe and conservative max-min objective, optimising for the performance in the worst-case environment~\cite{DBLP:conf/nips/CollinsMS20}, to optimising for risk-based measures exploiting the probabilistic nature of upMDPs, such as VaR or CVaR~\cite{DBLP:conf/nips/GreenbergMCM23}. 
Crucially, though, none of the existing learning algorithms are able to provide theoretical guarantees as to the performance of the generated policies in unseen environments;
this is a core contribution of our framework.

An overview of our approach is illustrated in Figure~\ref{fig:procedure}.
We assume access to multiple sampled environments $\mathcal{M}[\theta_i]$,
each of which is an MDP induced by a parameter valuation $\theta_i$ from the unknown distribution $\mathbb{P}$.
These are not fully known; instead we are able to access a set of sample trajectories from each one.
In our UAV example, this equates to taking the drone outside on a new day and encountering a new set of environmental conditions (or a simulation of this).

% We are facing two levels of uncertainty: (1) we want to provide guarantees for a policy's performance over the entire unknown distribution $\mathbb{P}$, which we can only estimate from sample environments; (2) the sample environments themselves are not fully known and can only be approximated from sample trajectories.

%\textcolor{red}{[somewhere we should also emphasise that we do not need knowledge of big P]}
% Instead, we are only able to access the induced \emph{(partially) unknown MDP}:
% we know the model's underlying structure (and, potentially, dependencies between
% some transition probabilities that share parameters), but the exact dynamics are
% unknown and can only be inferred from sample trajectories.

% to be maximised
% (e.g., the probability of mission completion, or the expected total reward).
% So, the value of $J$ for each policy $\pi$ is a random variable with
% respect to the parameter distribution $\mathbb{P}$ (see Figure~\ref{fig:inducedVal}).
% We want policies that are provably robust to the variation in parameter values.
% More precisely, we aim to provide a \emph{performance guarantee} $\smash{\tilde{J}}$
% which has a low \emph{violation risk} $\smash{r(\pi,\tilde{J})}$,
% i.e., a low probability that the evaluation function does not meet the threshold $\tilde{J}$
% (see Figure~\ref{fig:violationrisk}).
Our framework divides the sample environments into two groups, a \emph{training set} and a \emph{verification} set.
The training set is used to learn a robust policy.
For this we build on existing IMDP-based policy learning methods
and also consider robust meta reinforcement learning techniques.
% passed to an arbitrary robust policy learning algorithm~\cite{DBLP:conf/nips/CollinsMS20,DBLP:conf/nips/GreenbergMCM23,DBLP:journals/corr/Rickard23}.
%such as robust meta RL \textcolor{red}{[only, any other alternatives?]}. 
The verification set is used to derive the guarantees on the performance of the robust policy obtained from the training set. 
For this, we apply PAC IMDP learning to each unknown MDP in the verification set.
%We employ the computed policy and apply \emph{statistical model checking} (SMC)~\cite{DBLP:journals/tomacs/AghaP18,DBLP:conf/vmcai/HeraultLMP04,DBLP:conf/rv/LegayDB10} techniques to each \textcolor{red}{verification MDP [rephrase]}. 
%In this paper, we assume finite MDPs with a known graph structure and apply \emph{Model-SMC}~\cite{DBLP:conf/cav/AshokKW19,DBLP:journals/corr/meggendorfer24}. 
Concretely, we use sample trajectories to infer IMDP overapproximations, which contain the true, unknown MDPs with a user-specified confidence. From those, we can derive lower bounds on the performance $J$ of the learned policy in each of the environments, which hold with the specified confidence. 

% We sample trajectories from partially unknown MDPs $\mathcal{M}[\theta_i]$
% and use these to build approximations of the corresponding sample environments
% as \emph{interval MDPs} (IMDPs)~\citep{DBLP:journals/mor/WiesemannKR13}.
% Sample environments are divided into a \emph{training set}, for learning a robust policy,
% and a \emph{verification set} for evaluating its performance guarantee $\tilde{J}$.
% % \dave{commented out ``combined overapproximation'' for now}
% % The robust policy is obtained from a combined overapproximation of all learned IMDPs.
% %This guarantee leverages the confidence in IMDP over-approximations to include the underlying sampled MDPs~\cite{DBLP:conf/icml/StrehlL05}.
% For the latter, we produce a probably-approximately correct (PAC) guarantee, giving a bound $\varepsilon$ on the risk of not achieving performance $\tilde{J}$ in unseen environments, with high user-specified confidence $1-\eta$.
% We also allow tuning of the trade-off between the performance guarantee and violation risk via discarding of some samples.
% % \dave{not yet mentioned: possibility to incorporate model knowledge and apply model-based optimisations}
% %
% An evaluation of our approach on a range benchmarks
% shows that we can generate highly performing and robust policies,
% plus probabilistic guarantees that quantify their performance.

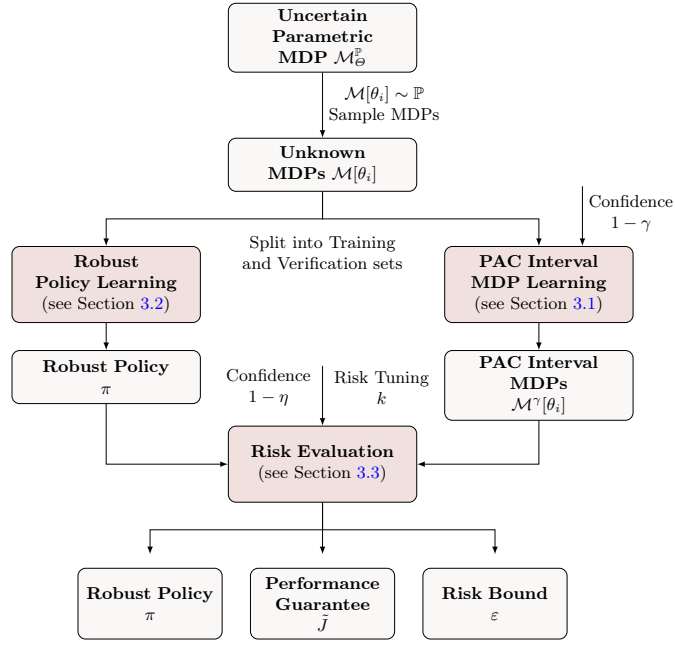
\begin{figure}[ht]
    \centering
    \usetikzlibrary{shapes,arrows,positioning}
\definecolor{lightred}{RGB}{241, 225, 222}

\scalebox{0.72}{
\begin{tikzpicture}[node distance=3cm, auto]

% Define block styles
\tikzstyle{block} = [rectangle, rounded corners, draw, fill=offwhite, text centered, minimum height=3em, inner sep= 3pt]
\tikzstyle{block2} = [rectangle, rounded corners, draw, fill=lightred, text centered, minimum height=4.3em, inner sep= 3pt]
\tikzstyle{line} = [draw, -latex]

% Place nodes
\node [block, text width=10em] (start) {\textbf{Uncertain Parametric\\MDP} $\mathcal{M}_{\Theta}^{\mathbb{P}}$};

\node [block, below= 1.2 and 0 of start, text width = 10em] (process1) {\textbf{Unknown\\MDPs} $\mathcal{M}[\theta_i]$}; %Partially

\node [block2, below right=1 and .5 of process1, text width = 10em] (process2) {\textbf{PAC Interval \\ MDP Learning}\\(see~Section~\ref{sec:paclearning})};

\node [block2, below left=1 and .5 of process1, text width = 10em] (process3) {\textbf{Robust\\Policy Learning} \\ (see~Section~\ref{sec:policylearning})};

\node [block, below= .5 and 0 of process3, text width = 10em] (policy) {\textbf{Robust Policy} \\$\pi$};

\node [block, below= .5 and 0 of process2, text width = 10em] (pacimdps) {\textbf{PAC Interval\\ MDPs}\\ $\mathcal{M}^{\gamma}[\theta_i]$};

\node [block2, below right =.4 and .5 of policy, text width = 10em] (riskeval) {\textbf{Risk Evaluation} \\ (see~Section~\ref{sec:riskeval})};

\node[below = 0.7 and 0 of process1, text width = 12em, text centered]{\small Split into Training\\and Verification sets};

% Draw edges
\path [line] (start) -- (process1) node[midway, right, align=center] {\footnotesize$\mathcal{M}[\theta_i] \sim \mathbb{P}$\\ \footnotesize Sample MDPs};
\path [line] (process1.south) -- ++(0,-.5) -| (process3.north);
\path [line] (process1.south) -- ++(0,-.5) -| (process2.north);

\path [line] (process2) -- (pacimdps) node[midway, right, align=center] {};

\path [line] (process3) -- (policy) node[midway, right, align=center] {};

\path [line] (policy.south) |- (riskeval.west);
\path [line] (pacimdps.south) |- (riskeval.east);

\path [line] ([yshift = 1.1cm, xshift=.8cm]process2.north) -- ([xshift = .8cm]process2.north) node[align = center, right, xshift = 0cm, yshift = 0.6cm]{\small Confidence\\\small $1-\gamma$};

\path [line] ([yshift = 1.1cm]riskeval.north) -- (riskeval.north) node[align = center, left, xshift = -.1cm, yshift = 0.7cm]{\small Confidence\\\small $1-\eta$} node[align = center, right, xshift = .1cm, yshift = 0.7cm]{\small Risk Tuning\\\small  $k$};

\node[below left = 1.2 and 0.1 of riskeval, text width = 7.5em, text centered, block, minimum height = 4em](robpol) {\textbf{Robust Policy}\\ $\pi$};

\node[below = 1.2 and 0 of riskeval, text width = 7.5em, text centered, block, minimum height = 4em] (perfgua) {\textbf{Performance\\ Guarantee}\\ $\tilde{J}$};

 \node[below right = 1.6 and -1.2 of riskeval, text width = 12em, text centered] (plus){\textbf{$ $}};

\node[below right = 1.2 and 0.1 of riskeval, text width = 7.5em, text centered,block, minimum height = 4em] (riskbound) {\textbf{Risk Bound}\\ $\varepsilon$};

\path [line] (riskeval.south) -- ++(0,-.5) -| ([yshift=.25cm]robpol.north);
\path [line] (riskeval.south) -- ++(0,-.5) -| ([yshift=.25cm]riskbound.north);
\path [line] (riskeval.south) -- ([yshift=.25cm]perfgua.north);

\end{tikzpicture}
}
    \caption{Overview of our framework to derive performance and risk guarantees for  policies learned on upMDPs. The setup includes two layers of uncertainty: we sample and analyse unknown environments from an unknown distribution.}
    \label{fig:procedure}
    \hfill
    \vspace{-14pt}
\end{figure}

To tackle the higher layer of uncertainty and infer a bound for the policy's robust performance over the entire unknown distribution $\mathbb{P}$ underlying the sample MDPs, we develop a new approach based on \emph{scenario optimisation}~\cite{DBLP:journals/siamjo/CampiG08,DBLP:journals/jota/CampiG11}.
% Scenario optimisation (or the \emph{scenario approach})
This takes samples of the performance $J$ and a user-specified performance bound $\tilde{J}$ and provides a PAC guarantee on the probability of the performance on a new sample being less than $\tilde{J}$, i.e., the \emph{risk}. However, in our setup we do not obtain samples of $J$ directly, but derive lower bounds from the learned PAC IMDPs, which only hold up to a certain confidence. Our key theoretical contribution, presented in Section~\ref{sec:riskeval}, is a generalisation of the scenario approach that can handle samples whose values are only known to lie in a confidence interval.

Our theoretical results combine the two layers of uncertainty: (1) the finite sampling of MDPs from the distribution $\mathbb{P}$, (2) the fact that sampled MDPs are unknown, so the performances of the learned policy are only inferable up to a certain confidence. The result is a single PAC guarantee on the policy's performance which holds with a high, user-specified confidence.

Furthermore, our framework allows tuning of the trade-off between performance guarantee and risk. By excluding the $k$ worst-case sample environments, users can discard unlikely outliers, resulting in a higher performance guarantee at the cost of an increased risk bound, adjustable to the level the user considers admissible.
We implement our framework as an extension of the PRISM model checker~\cite{DBLP:conf/cav/KwiatkowskaNP11} and show that it can tightly quantify the performance and associated risk of learned policies on a range of benchmarks.

%We remark that our results seamlessly carry over to other (Property-)SMC techniques~\cite{DBLP:conf/cav/AshokKW19,DBLP:journals/corr/meggendorfer24} for analysing the sample MDPs, which do not require finite MDPs or known graph structures. 

In summary, our contributions are:
(1) a novel framework and techniques for producing certifiably robust policies in uncertain parametric MDPs for which both the parameters and transition probability functions are unknown;
(2) new theoretical results which yield PAC guarantees on a policy's robust performance on unseen environments, where sample environments are unknown and can only be estimated from trajectories;
(3) an implementation and evaluation of the framework on a range of benchmarks.

% We summarise our contributions as follows:
%     (1) we present the first PAC guarantee for a policy's robust performance on unseen environments of an upMDP, where sample environments are unknown and can only be estimated from trajectories;
%     (2) we extend the scenario approach to uncertain samples with values only known in probability;
%     (3) we allow for a user-adjustable risk-performance trade-off to tune the resulting guarantees;
%     (4) we implement our approach inside the PRISM model checker and demonstrate that it tightly quantifies the performance and associated risk of learned policies on upMDPs. 

\ifthenelse{\isundefined{\techreport}}{%
An extended version of this paper, with full proofs of all results,
extended experiments and further technical details, can be found in \cite{arxiv}.
}{%
This paper is an extended version of \cite{SAP25},
including full proofs of all results,
extended experiments and further technical details.
}

\subsection{Related Work}

Epistemic uncertainty in MDPs has received broad attention across many areas, including formal methods, planning and reinforcement learning~\cite{badings2023decisionmaking}.
As mentioned above, there are various ways to model this using uncertain MDPs~\citep{DBLP:journals/ai/GivanLD00,DBLP:journals/mor/WiesemannKR13}.
There are also techniques such as robust dynamic programming to synthesise robust policies for these models~\cite{DBLP:journals/mor/Iyengar05,DBLP:journals/ior/NilimG05}
and approaches to learn the models from trajectory data~\cite{DBLP:conf/cav/AshokKW19,DBLP:journals/corr/meggendorfer24,DBLP:conf/icml/StrehlL05,DBLP:conf/nips/SuilenS0022}.
In this work, however, we investigate \emph{parametric} uncertainty sets with unknown distributions over parameter valuations.
%We exploit stochasticity to allow for a permissible risk of suboptimality in exchange for higher performance guarantees.

% \dave{other IMDP/scenario, e.g. \cite{BRA+23}?}

Uncertain parametric MDPs have emerged as a common model in meta reinforcement learning~\cite{DBLP:conf/nips/CollinsMS20,DBLP:conf/icml/FinnAL17,ghosh2021generalization,DBLP:conf/nips/GreenbergMCM23,DBLP:conf/nips/GuptaMLAL18} and gained attention in formal methods~\cite{DBLP:journals/sttt/BadingsCJJKT22,DBLP:journals/corr/Rickard23}. On the one hand, meta reinforcement learning trains policies on multiple unknown environments sampled from an upMDP, using policy gradient methods~\cite{DBLP:conf/nips/SuttonMSM99}, in order to generalise to unseen environments. However, to our knowledge, none of these algorithms provide theoretical generalisation guarantees, either on their average~\cite{DBLP:conf/icml/FinnAL17,DBLP:conf/nips/GuptaMLAL18} or their robust performance~\cite{DBLP:conf/nips/CollinsMS20,DBLP:conf/nips/GreenbergMCM23}. 

On the other hand, existing formal methods approaches to upMDPs do not offer the generality of meta RL setups. The work in~\cite{DBLP:journals/sttt/BadingsCJJKT22} uses scenario methods and provides PAC guarantees, but for the existence of a policy that achieves a certain performance, not robust policy synthesis; they also require full knowledge of sampled parameter valuations and environments. In \cite{DBLP:journals/corr/Rickard23}, concrete robust policies are synthesised with a PAC guarantee for performance on unseen environments, but this also relies on complete knowledge of all sampled valuations, reducing it to a special case of our approach.
In our work, we address the very general problem of unknown sample environments;
we target the scalability and generality of meta RL, while providing formal guarantees that are independent of the model size and previously unattainable for policy training methods like those in~\cite{DBLP:journals/corr/Rickard23}.

Also related is \cite{Costen_Rigter_Lacerda_Hawes_2023} which uses parametric MDPs in a Bayesian setting;
parameter valuations are unknown but the model's transition functions are known and assumed to be defined by polynomial expressions.
Other recent work in~\cite{DBLP:conf/birthday/ChiLT0J25} combines, like us, the scenario approach and parametric MDPs, but for a different setting that assumes uniform distributions over parameter spaces.
We also mention~\cite{DBLP:conf/cav/BadingsJJSV22}, which uses scenario optimisation with imprecise constraints to analyse continuous-time Markov chains with uncertain transition rates.

\section{Preliminaries}
\label{sec:prelim}
We review the key formalisms used in our approach. Let $\Delta(S) = \{\mu \colon S \to [0,1] \mid \sum_{s} \mu(s) = 1\}$ denote the set of all probability distributions over a finite set $S$.

\begin{definition}[Parametric MDP]
A \emph{parametric Markov decision process} (pMDP) is a tuple $M_{\Theta} = (S, s_I, A, P_{\Theta})$, where $S$ and $A$ are finite state and action spaces, $s_I \in \Delta(S)$ is the initial state distribution, and $P_{\Theta} \colon \Theta \times S \times A \to \Delta(S)$ is the parametric transition probability function over the parameter space $\Theta$. Fixing a valuation $\theta \in \Theta$ induces a standard MDP $\mathcal{M}_{\Theta}[\theta]$, or $\mathcal{M}[\theta]$ for short, with transition kernel $P_\theta \colon S \times A \to \Delta(S)$ defined as $P_\theta(s,a,s') = P_\Theta(\theta,s,a,s')$. \demo
\end{definition}

% \begin{figure}
%     \centering
%     \input{text/mdps}
%     \caption{pMDP, instantiated true but partially unknown MDP, and learned IMDP. TODO: unclutter and add sublabeling}
%     \label{fig:pdmps}
% \end{figure}

    % Draw the arrows between the subfigures
    % \begin{tikzpicture}[overlay, remember picture]
    %     % Arrow between the first and second subfigure
    %     \draw[-stealth, thick] ([yshift=1ex]current bounding box.west |- current bounding box.center) ++(4.2,0) -- ++(3,0) node[midway, above] {\footnotesize$(p,q,r) \sim \mathbb{P}$};

    %     % Arrow between the second and third subfigure
    %     \draw[-stealth, thick] ([yshift=1ex]current bounding box.west |- current bounding box.center) ++(11.5,0) -- ++(3,0) node[midway, above] {\scriptsize IMDP Learning};
    % \end{tikzpicture}

% \begin{example}
%    Figure~\ref{fig:updmpa} shows a simple 4-state pMDP with 3 parameters.
    % Furthermore, the drone-delivery problem in Figure~\ref{fig:drone} can be suitably modelled as a pMDP, where each state represents a discrete location in the 3-dimensional environment, and the outcome of chosen actions depends on parameters such as the wind.
% \end{example}
%
% \begin{remark}
    Parametric MDPs can be seen as an abstract model for a \emph{set} of MDPs, i.e., they represent the instantiations induced by all possible valuations $\theta \in \Theta$. They are closely related to the model class of uncertain MDPs~\citep{DBLP:journals/ior/NilimG05,DBLP:journals/mor/WiesemannKR13}, in which each transition is associated with (potentially interdependent) sets of possible values.
% \end{remark}

 % We are interested in learning policies and quantifying their performance in parametric models, where parameters are subject to change according to some probability distribution.

\begin{definition}[Uncertain Parametric MDP]
        An \emph{uncertain parametric Markov decision process} (upMDP) $\mathcal{M}_{\Theta}^{\mathbb{P}} = (\mathcal{M}_\Theta, \mathbb{P})$  is a pMDP $\mathcal{M}_\Theta$ with a (potentially unknown) probability measure $\mathbb{P}$ over the parameter space $\Theta$. \demo
\end{definition}

% Our procedure is based on finite sampling and does not require any knowledge of the underlying set of parameters $\Theta$, the concrete parametric structure of the pMDP, or the probability distribution $\mathbb{P}$. Our only assumption is that the graph structure underlying the upMDP is known.

We assume that upMDPs are \emph{graph preserving}, meaning that induced MDPs share a common topology: \(\forall s, s' \in S, a \in A \colon (\forall \theta \in \supp(\mathbb{P}) \colon P_\theta(s, a, s') = 0) \lor (\forall \theta \in \supp(\mathbb{P}) \colon P_\theta(s, a, s') > 0)\). Although not strictly required, this assumption can be crucial for efficiently solving learned IMDP approximations of the induced MDPs~\cite{DBLP:journals/corr/meggendorfer24,DBLP:journals/mor/WiesemannKR13}.
We describe in Section~\ref{sec:opt} how to lift this assumption by resorting to techniques which only approximate the performance and not the model.

% \begin{wrapfigure}{r}{0.5\textwidth}
%     \centering
%     \includegraphics[width=\textwidth]{Figures/inducedplot.pdf}
% \caption{Probability measure $\theta \sim \mathbb{P}$ over parameters $\Theta$ (left) and value functions $V$ for fixed policies $\pi_1$ and $\pi_2$ as a random variables with measure $\mathbb{P}$ over the underlying parameters (right).}
% \label{fig:inducedValueFunc}
% \end{wrapfigure}

% \begin{figure}[h!]
%     \centering
%     \begin{minipage}{0.45\textwidth}
%         \centering
%         \includegraphics[width=0.87\textwidth]{Figures/inducedplot2.pdf} % replace with your image file
%         \caption{Probability measure $\theta \sim \mathbb{P}$ over parameters $\Theta$ and evaluation functions $J$ for fixed policies $\pi_1$ and $\pi_2$ as a random variables with measure $\mathbb{P}$ over the underlying parameters.}
%         \label{fig:inducedVar}
%     \end{minipage}
%     \hfill % Space between the figures
%     \begin{minipage}{0.45\textwidth}
%         \centering
%         \input{text/sample_risk_figure} % replace with your image file
%         \caption{This is the risk, i.e., that we are trying to bound. We infer the bound from samples.}
%         \label{fig:distriskplot}
%     \end{minipage}
% \end{figure}
% % \textcolor{red}{[can we better justify (perhaps `in practice') why this is `minor'?]}
% Knowledge of the graph structure is the minimal requirement on which we base our model-based approaches for policy learning and risk evaluation. However, we later comment on how this assumption can also be relaxed (see Section~\ref{sec:conc}).

Policies resolve action choices in MDPs and upMDPs. They are mappings \(\pi \colon (S \times A)^* \times S \to \Delta(A)\), assigning a distribution over actions based on (finite) histories of states and actions. In this work, we focus on synthesising memoryless policies \(\pi \colon S \to \Delta(A)\). While our theoretical results apply to arbitrary policy classes, learning and evaluating more expressive policies, such as those with finite memory, can be computationally expensive. We elaborate on this in Section~\ref{sec:opt}.

%We formalise the performance of a policy on instantiations of the upMDP as evaluation functions.
\begin{definition}[Evaluation Function]
For an upMDP \(\mathcal{M}_{\Theta}^{\mathds{P}} = (\mathcal{M}_{\Theta}, \mathds{P})\), an \emph{evaluation function} \(J \colon \Pi \times \Theta \to \mathbb{R}\) maps a policy \(\pi\) and a parameter valuation \(\theta\) to a performance value. We also denote by \(J(\pi, \mathcal{M})\) the evaluation of policy \(\pi\) on an arbitrary MDP \(\mathcal{M}\), such that \(J(\pi, \theta) = J(\pi, \mathcal{M}[\theta])\).
\end{definition}

%Our approach is agnostic to the type of the employed evaluation function. 
% For a policy $\pi$ resolving the nondeterministic choices in an MDP $\mathcal{M}$
Typical evaluation functions include the reachability probability ${\bm\Pr}^{\pi}_{\mathcal{M}}(\lozenge T)$ of eventually reaching a set of target states $T \subseteq S$, the reach-avoid probability ${\bm\Pr}^{\pi}_{\mathcal{M}}(\neg C \until T)$ of reaching states in $T$ while not entering a set of avoid states $C \subseteq S$, the expected reward $\mathbb{E}_{\mathcal{M}}^\pi(\lozenge T)$ accumulated before reaching a set $T$, given a reward structure, %~\cite{DBLP:conf/formats/AndovaHK03},
expected accumulated reward over finite or infinite time horizons, %~\cite{,DBLP:books/lib/SuttonB98},
or more sophisticated properties expressed in temporal logics such as LTL~\cite{DBLP:conf/focs/Pnueli77} or PCTL~\cite{DBLP:journals/fac/HanssonJ94}. Throughout this work's technical presentation, we assume that performance $J$ is maximised. By changing the directions of optimisation and inequalities, our results also directly apply to the dual minimisation case.

% Commonly used evaluation functions are, e.g., (un-)discounted expected reward~\citep{DBLP:books/lib/SuttonB98} or the probability of satisfying requirement over trajectories~\citep{DBLP:books/daglib/baierkatoen}. 
% \textcolor{red}{[so in practice, we do both optimisation of performance and of temporal requirements]}
For a fixed policy $\pi$ and upMDP $\upmdp$, the evaluation function $J$ only depends on the valuations $\theta$. Hence, $J$ is a random variable with measure $\mathbb{P}$ (see Figure~\ref{fig:inducedVal}). 
The \emph{violation risk} is the probability that policy $\pi$ achieves a value less than a stated performance guarantee $\tilde{J}$ on $\upmdp$ (see Figure~\ref{fig:violationrisk}).
\begin{definition}[Violation Risk]
    \label{def:risk}
    The \emph{violation risk} of policy $\pi$ for performance guarantee $\tilde{J} \in \mathbb{R}$, denoted by $r(\pi, \tilde{J})$, is defined as:
    \begin{equation}
        r(\pi, \tilde{J}) = \mathbb{P} \left\{ \theta \in \Theta \colon J(\pi , \theta) < \tilde{J} \right\}. 
    \end{equation}
\end{definition}
\noindent
There is an inherent trade-off between violation risk and performance guarantee: a higher guarantee is associated with a higher risk, regardless of $\upmdp$ and $J$. 

The framework we present in this paper is based on learning and solving approximations of MDPs in the form of \textit{interval MDPs}~\cite{DBLP:journals/ai/GivanLD00,DBLP:journals/mor/WiesemannKR13}.

% Our approach addresses two layers of uncertainty: the finite sampling of MDPs from $\upmdp$, and the incomplete knowledge of the samples. To tackle this, we construct and solve model approximations in the form of \textit{interval MDPs}~\cite{DBLP:journals/tcs/HaddadM18,DBLP:journals/mor/WiesemannKR13}.
% We analyze sample trajectories from these partially unknown MDPs to estimate the unknown parameters with intervals, resulting in \textit{interval MDPs}.

\begin{definition}[Interval MDP]
An interval Markov decision process (IMDP) \(\mathcal{M}^I = (S, s_I, A, P^{I})\) is an MDP with a probabilistic interval transition function \(P^I \colon S \times A \times S \to \mathbb{I}\), where \(\mathbb{I} = \{[a,b] \mid 0 < a \leq b \leq 1\} \cup \{[0,0]\}\) is the set of all graph-preserving intervals. We say that IMDP \(\mathcal{M}^I\) \emph{includes} MDP \(\mathcal{M}\), denoted by \(\mathcal{M} \in \mathcal{M}^I\), if $\mathcal{M}$ and $\mathcal{M}^I$ share the same state and action spaces, and \(P(s,a,s') \in P^I(s,a,s')\) for all \(s,s' \in S, a \in A\). \demo
\end{definition}

For IMDPs, we typically adopt a \emph{robust} view of a policy's performance, i.e., the \emph{worst-case} (minimum) value over any included MDP.
We lift the evaluation function \(J\) to an IMDP \(\mathcal{M}^I\) as follows:
\begin{equation}
    \label{eqn:imdpvalueineq}
    J(\pi, \mathcal{M}^I) = \min_{\mathcal{M} \in \mathcal{M}^I} J(\pi, \mathcal{M}) \implies J(\pi, \mathcal{M}^I) \leq J(\pi, \mathcal{M}) \mbox{ for all } \mathcal{M} \in \mathcal{M}^I.
\end{equation}
For key classes of properties used here (e.g., reachability probabilities, rewards),
this value can be obtained via robust dynamic programming~\cite{DBLP:journals/mor/Iyengar05,DBLP:journals/ior/NilimG05,DBLP:journals/mor/WiesemannKR13}.

\section{Learning Certifiably Robust Policies for upMDPs}

This section presents our framework for computing performance and risk guarantees for learned policies in uncertain parametric MDPs. An overview of this framework was illustrated in Figure~\ref{fig:procedure} and introduced in Section~\ref{sec:intro}. 
We assume a fixed upMDP \(\upmdp\) and access to a sample set 
\(\mathcal{D} = \{\mathcal{M}[\theta_i] \mid \theta_i \sim \mathbb{P}\}\) of unknown MDPs. This sample set is split into disjoint training and verification sets. The training set is used to compute a robust policy \(\pi\), which is then evaluated on the verification set to derive a performance guarantee \(\tilde{J}\) and bound the violation risk when the policy is deployed in an unseen environment sampled from \(\mathbb{P}\). 
The overall goal is formally stated as follows.

% which builds IMDP over-approximations of the unknown sample MDPs 

% for estimating a policy's performance in unknown MDPs sampled from the unknown distribution. It leverages the learning of PAC IMDP over-approximations of the sampled environments which contain the true MDP with high confidence. Section

% In this paper, we employ model SMC~\cite{DBLP:journals/corr/meggendorfer24} to approximate sampled unknown environments by analysing trajectories. Using these learned estimates, we evaluate a policy's performance and bound its risk when deployed in unseen environments.
% The overall problem can be stated as follows:
%Our approach addresses the following formal problem:
 
% Our procedure learns IMDP estimates of the sampled, unknown environments and uses them to construct robust policies. We leverage confidence in these estimates to provide a PAC guarantee on the violation risk of the learned policy on the underlying, unknown upMDP, solving the following formal problem:

\begin{myproblem}
    \label{prob:1}
Given a upMDP \(\upmdp\) with unknown parameter distribution $\mathbb{P}$, an evaluation function \(J\), and a confidence level \(\eta > 0\),
%. From a set of MDP samples \(\mathcal{D}\),
find a robust policy \(\pi\), a (tight) performance guarantee \(\tilde{J}\), and a (tight) risk bound \(\varepsilon > 0\), such that: %$\Pr \{ r(\pi, \tilde{J}) \leq \varepsilon \} \geq 1 - \eta.$
\[
    {\Pr}\left\{ r(\pi, \tilde{J}) \leq \varepsilon \right\} \geq 1 - \eta.
\]
\end{myproblem}
There exist trivial solutions to Problem~\ref{prob:1}, such as selecting a very high risk bound or a very low performance guarantee. We aim to identify a tight solution that maximizes the performance guarantee with a precise, low risk bound.

The core part of our framework is the means to establish these performance and risk bounds for policies evaluated in unknown environments.
From each unknown MDP in the verification set, we sample trajectories to learn an IMDP approximation that includes the MDP with high confidence. Solving these IMDPs yields probabilistic bounds on the policy's performance in each environment.

Our main theoretical results build upon scenario optimisation~\cite{DBLP:journals/siamjo/CampiG08,DBLP:journals/jota/CampiG11}, the principal challenge being to incorporate the additional layer of uncertainty introduced by only being able to estimate the policy's performance in each unknown sampled environment.
The result is a single PAC guarantee on the policy's performance in unseen environments, % from the upMDP,
stated in Problem~\ref{prob:1} above.

The process of establishing these guarantees is agnostic to the manner in which policies are produced.
We consider two approaches, first taking a novel combination of existing methods for IMDPs and upMDPs~\cite{DBLP:journals/corr/Rickard23,DBLP:conf/nips/SuilenS0022},
and secondly adopting a gradient-based technique from robust meta reinforcement learning~\cite{DBLP:conf/nips/GreenbergMCM23}. 

The remainder of the section is structured as follows.
Since PAC learning of IMDPs is used in multiple places, we discuss this first, in Section~\ref{sec:paclearning}.
Section~\ref{sec:policylearning} covers robust policy learning,
Section~\ref{sec:riskeval} presents our main theoretical contributions
and Section~\ref{sec:opt} describes several optimisations and extensions.

\vspace*{-1em}
\subsection{PAC IMDP Learning of Unknown MDPs}
\label{sec:paclearning}

We follow established approaches for PAC learning of IMDP approximations introduced in~\cite{DBLP:conf/cav/AshokKW19,DBLP:journals/corr/meggendorfer24,DBLP:conf/icml/StrehlL05,DBLP:conf/nips/SuilenS0022}. Consider an unknown MDP $\mathcal{M}[\theta_i]$. We assume access to trajectories from $\mathcal{M}[\theta_i]$, consisting of sequences of triples $(s, a, s')$ representing states, chosen actions and successor states. Leveraging the Markov property of MDPs, we treat each triple as an independent Bernoulli experiment to estimate the transition probability to state $s'$ from $s$ when choosing action $a$. We denote the number of times action $a$ was chosen in state $s$ across all sample trajectories as $\#(s,a)$, and the number of times this choice led to $s'$ as $\#(s,a,s')$. The resulting point estimate for the transition probability is thus given by:
\begin{equation}
    \label{eqn:pointest}
    \tilde{P}(s,a,s') = \frac{\#(s,a,s')}{\#(s,a)},
\end{equation}
for $\#(s,a) > 0$. We construct an IMDP by transforming the point estimates into PAC  confidence intervals~\cite{DBLP:books/daglib/boucheron13}. Traditionally, this is achieved with Hoeffding's inequality~\cite{Hoeffding1994,DBLP:conf/icml/StrehlL05}. Recent work has demonstrated that significantly tighter model approximations can be obtained by employing inequalities tailored to the binomial distribution, such as the Wilson score interval with continuity correction~\cite{DBLP:journals/corr/meggendorfer24,Wilson27}.

Let $1-\gamma$ with $\gamma > 0$ be the desired confidence level for $\mathcal{M}[\theta_i]$ to be included in the IMDP, which we denote $\mathcal{M}^\gamma[\theta_i]$. This confidence is distributed over all $n_u$ unknown transitions as $\gamma_P = \gamma / n_u$.
%, where is the number of unknown transitions in $\mathcal{M}[\theta_i]$. 
Let $H = \#(s,a)$ and $\tilde{p} = \tilde{P}(s,a,s')$. For each unknown transition, the transition probability interval is given by:
\begin{equation}
    \label{eqn:intervalest}
    P^{\gamma}(s,a,s') = [\max(\mu,  \underline{p}_{wcc}), \min(\overline{p}_{wcc}, 1)],
\end{equation}
with:
\begin{equation}
    \underline{p}_{wcc} = \left(2H\tilde{p} + z^2 - z\sqrt{z^2-\frac{1}{H}+4H\tilde{p}(1-\tilde{p}) + 4\tilde{p} - 2}-1\right)/ \,(2(H+z^2)),
\end{equation}
\begin{equation}
    \overline{p}_{wcc} = \left(2H\tilde{p} + z^2 + z\sqrt{z^2-\frac{1}{H}+4H\tilde{p}(1-\tilde{p}) - 4\tilde{p} + 2} + 1\right)/ \,(2(H+z^2)), 
\end{equation}
where $z$ is the $1- \frac{\gamma_P}{2}$ quantile of the standard normal distribution~\cite{DBLP:journals/corr/meggendorfer24} and $\mu > 0$ is an arbitrarily small quantity to preserve the known graph structure. For unvisited state action pairs with $\#(s,a) = 0$, we set $P^{\gamma}(s,a,s') = [\mu, 1]$, for all $s'$ in the known support. $P^{\gamma}(s,a,s')$ contains the true transition probability $P(s,a,s')$ with a confidence of at least $1-\gamma_P$. By applying a union bound over the unknown transitions, we obtain the following overall guarantee:

\begin{lemma}[\cite{DBLP:journals/corr/meggendorfer24}]
    The true, unknown MDP $\mathcal{M}[\theta_i]$ is contained in its IMDP overapproximation $\mathcal{M}^\gamma[\theta_i]$ with probability at least $1-\gamma$. \qed
    \label{lem:inclusion}
\end{lemma}

%This procedure estimates sampled environments from the unknown distribution with a desired level of confidence. 
The confidence in the approximation of each environment is independent of the number or length of the trajectories analysed. However, more or longer trajectories generally lead to higher state-action counts, resulting in tighter intervals. In Section~\ref{sec:exp}, we examine how the number of trajectories analysed influences the tightness of our performance guarantee and the associated risk.

\subsection{Robust Policy Learning}
\label{sec:policylearning}

We consider two distinct approaches to robust policy learning: \emph{robust IMDP policy learning} and \emph{robust meta reinforcement learning}.
For the former, we propose a combination of techniques for robust policy synthesis for upMDPs with access to fully known sample environments~\cite{DBLP:journals/corr/Rickard23} and IMDP learning for single unknown environments~\cite{DBLP:conf/nips/SuilenS0022}.
For the latter, we adopt a class of algorithms that optimise a policy's robust performance using policy gradient-methods~\cite{DBLP:conf/nips/CollinsMS20,DBLP:conf/nips/GreenbergMCM23}.

\startpara{Robust IMDP Policy Learning}
Similarly to PAC IMDP learning in Section~\ref{sec:paclearning}, we use sample trajectories to compute an IMDP overapproximation~\cite{DBLP:conf/nips/SuilenS0022} for each unknown MDP in the training set.
Then, like in~\cite{DBLP:journals/corr/Rickard23}, we combine models across the training set to perform policy synthesis. To obtain a policy that is robust across all samples, the learned IMDPs are combined by merging the transition intervals of each IMDP as $[a,b] \sqcup [c,d] = [\min(a,c), \max(b,d)]$. The resulting IMDP over-approximates all training MDPs, and the corresponding optimal deterministic policy considers the worst-case probability for each transition~\cite{DBLP:journals/mor/WiesemannKR13}. 
As the IMDPs for the training set are only used for policy synthesis and not for formal risk or performance analysis, we are not restricted to PAC IMDP learning. We can leverage a rich pool of interval learning algorithms, which, while lacking formal inclusion guarantees, provide empirically tighter intervals from fewer trajectories. A detailed overview and comparison of interval learning methods and their model approximation capabilities is conducted in~\cite{DBLP:conf/nips/SuilenS0022}. We evaluate the best-performing approaches in our benchmarks in Section~\ref{sec:exp}.

\startpara{Robust Meta Reinforcement Learning}
IMDP policies can be overly conservative, as they consider the worst-case scenario for each transition independently~\cite{DBLP:journals/corr/Rickard23}. Additionally, IMDP learning produces memoryless deterministic policies. While sufficient for optimality in IMDPs, there exist upMDPs where an optimal robust policy requires randomization or memory~\cite{DBLP:conf/fsttcs/RaskinS14}. 

\emph{Robust meta-reinforcement learning} (RoML) extends classical reinforcement learning from a single MDP to upMDPs~\cite{DBLP:conf/nips/CollinsMS20,DBLP:conf/nips/GreenbergMCM23}. RoML employs policy gradient methods~\cite{DBLP:conf/nips/SuttonMSM99} to optimize a policy's performance across sampled training environments, estimating performance from sampled trajectories in each unknown environment. Unlike standard meta-RL, which maximises expected rewards across environments~\cite{DBLP:journals/corr/Beck23} and often yields strong average but poor worst-case performance, RoML prioritises robustness. It trains a policy using RL techniques to optimise the performance in the worst-case environment (via the max-min objective)~\cite{DBLP:conf/nips/CollinsMS20} or in the $\alpha$-quantile of worst-case environments (via a risk-aware CVaR objective)~\cite{DBLP:conf/nips/GreenbergMCM23}. 
RoML differs from \emph{robust RL}~\cite{derman2019bayesian,wang2021online}, which focuses on a single uncertain MDP, seeking a policy that is robust to its internal uncertainty. In contrast, RoML addresses multiple MDPs sampled from an unknown distribution, aiming for a policy that generalizes to the full distribution by achieving a strong robust performance across the sampled MDPs. This aligns well with our setup, and we refer the reader to~\cite{DBLP:conf/nips/CollinsMS20,DBLP:conf/nips/GreenbergMCM23} for further details on this approach.

\subsection{Certifying Policy Performance and Risk in upMDPs}
\label{sec:riskeval}

We now present our main theoretical contributions for quantifying the performance and violation risk of a policy $\pi$ deployed in unseen environments of an upMDP. % $\upmdp$.
Specifically, we provide two results that derive PAC guarantees from lower bounds on $\pi$'s performance, which we obtain by building and analysing PAC IMDPs (see Section~\ref{sec:paclearning}) for the sampled environments that make up the verification set.

The first result extends the scenario approach~\cite{DBLP:journals/siamjo/CampiG08,DBLP:journals/jota/CampiG11} to incorporate additional uncertainty, as the lower bounds hold only with a certain confidence. This provides a PAC guarantee on the policy's performance, reasoning over the unknown distribution of true environments \(\mathbb{P}\), based solely on estimations from sampled unknown environments attained by IMDP learning.
The second result introduces flexibility in balancing the risk-performance trade-off. By extending a second result of the scenario approach, we allow the performance bound to be tuned by excluding worst-case outlier samples from the analysis, potentially leading to a higher performance guarantee at the cost of an increased risk bound.

Assume that the verification set comprises $N$ sampled MDPs,
from which we have learnt the PAC IMDPs $\{\mathcal{M}^\gamma[\theta_i]\}_{1 \leq i \leq N}$.
% , learned for the verification set of unknown MDPs sampled from $\mathbb{P}$.
This establishes probabilistic lower bounds on the policy $\pi$'s performance in the underlying unknown MDPs.
From Equation~\eqref{eqn:imdpvalueineq}, we have that $\mathcal{M}[\theta_i] \in \mathcal{M}^\gamma[\theta_i] \Rightarrow J(\pi, \mathcal{M}^\gamma[\theta_i]) \leq J(\pi, \mathcal{M}[\theta_i])$, where $J(\pi, \mathcal{M}^\gamma[\theta_i])$ can be obtained by standard solution methods for IMDPs, such as robust dynamic programming~\cite{DBLP:journals/mor/Iyengar05,DBLP:journals/ior/NilimG05}.
By Lemma~\ref{lem:inclusion}, it follows that:
\begin{equation}
    \label{eqn:probineq}
    \mathbb{P}\{J(\pi, \mathcal{M}^\gamma[\theta_i]) \leq J(\pi, \mathcal{M}[\theta_i])\} \geq 1 - \gamma.
\end{equation}
% We use the verification set from Section~\ref{sec:paclearning} to bound the violation risk of the learned policy $\pi$ when applied in an unseen environment. We analyse $\pi$ on the PAC IMDPs $\mathcal{V}_N^\gamma = \{\mathcal{M}^\gamma[\theta_i]\}_{1 \leq i \leq N}$ to bound the policy's performance on the underlying unknown MDPs. From Equation~\ref{eqn:imdpvalueineq}, we have $\mathcal{M}[\theta_i] \subseteq \mathcal{M}^\gamma[\theta_i] \Rightarrow J(\pi, \mathcal{M}^\gamma[\theta_i]) \leq J(\pi, \mathcal{M}[\theta_i])$. From Lemma~\ref{lem:inclusion}, it follows that
% \begin{equation}
%     \label{eqn:probineq}
%     \mathbb{P}\{J(\pi, \mathcal{M}^\gamma[\theta_i]) \leq J(\pi, \mathcal{M}[\theta_i])\} \geq 1 - \gamma.
% \end{equation}
To bound the violation risk, we formulate the problem as a convex optimisation problem with randomised constraints—a \textit{scenario program}~\citep{campi2018introduction}. We extend the generalisation theory of the \textit{scenario approach}~\cite{DBLP:journals/siamjo/CampiG08} to account for the uncertainty in the lower performance bounds (see Equation~\eqref{eqn:probineq}). The detailed formulation and derivation of our generalisation can be found in
\ifthenelse{\isundefined{\techreport}}{\cite{arxiv}}{Appendix~\ref{sec:appB}}.

% To obtain a bound on the violation risk, we formulate the problem as a convex optimisation problem with randomised constraints—a \textit{scenario program}~\citep{campi2018introduction}. We leverage the generalisation theory of the \textit{scenario approach}~\cite{DBLP:journals/siamjo/CampiG08}, adapting it to incorporate the uncertainty of the lower performance bounds (see Equation~\eqref{eqn:probineq}). We detail the formulation and the derivation of our generalisation in Appendix~\ref{sec:appB}.

% the generalisation from approximations of unknown samples to the entire unknown distribution in 

% performance values obtained on the learned IMDPs to under-approximate the performances on the underlying MDPs . 

\begin{figure}[t]
    \centering
    \begin{subfigure}[b]{0.48\textwidth}
        \centering
        \includegraphics[width=0.95\textwidth]{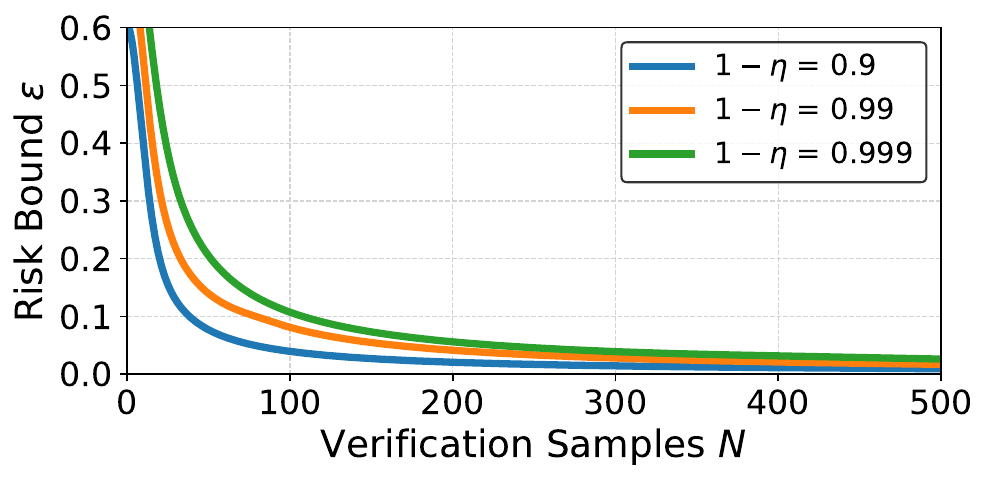}

        \label{fig:subfig1}
    \end{subfigure}
    \hfill
    \begin{subfigure}[b]{0.48\textwidth}
        \centering
        \includegraphics[width=0.95\textwidth]{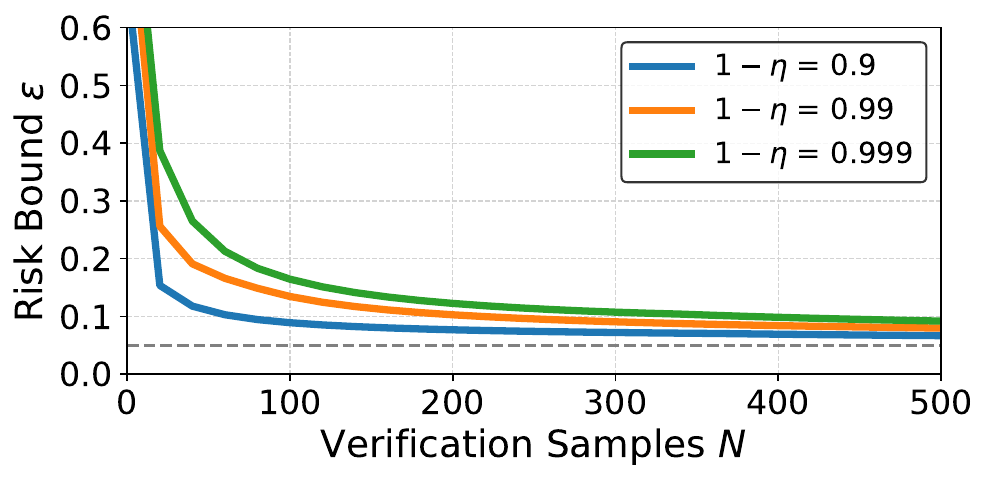}
        \label{fig:subfig2}
    \end{subfigure}
    \caption{Example risk bounds obtained from Theorem~\ref{thm:bound} (left) and Theorem~\ref{thm:bounddiscard} (right) for IMDP confidence $\gamma = 10^{-4}$. For Theorem~\ref{thm:bounddiscard}, 5\% of samples are discarded.}
    \label{fig:risknormalfull}
    \vspace{-5pt}
\end{figure}

\begin{theorem}
    \label{thm:bound}
    Given $N$ i.i.d.\ sample MDPs $\mathcal{M}[\theta_i]$ and IMDPs $\mathcal{M}^\gamma[\theta_i]$, such that $\mathbb{P}\{\mathcal{M}[\theta_i] \in \mathcal{M}^\gamma[\theta_i]\} \geq 1 - \gamma$. For any policy $\pi$ and confidence level $1 - \eta$, with $\eta > 0$, it holds that:
    \begin{equation}
        \mathbb{P}^N\left\{ r(\pi, \tilde{J}^{\gamma}) \leq \varepsilon(N, \gamma, \eta) \right\} \geq 1 - \eta,
    \end{equation}
    where $\tilde{J}^{\gamma} = \min_i J(\pi, \mathcal{M}^\gamma[\theta_i])$, and $\varepsilon(N, \gamma, \eta)$ is the solution to the following, for any $K \leq N$:
    \begin{equation}
        \sum_{i = K}^N \binom{N}{i}(1-\gamma)^i\gamma^{N-i} - (1 - \eta) = \sum_{i = 0}^{N-K} \binom{N}{i}\varepsilon^i(1-\varepsilon)^{N-i}. 
        \label{eqn:binom}
    \end{equation}
    % for any $M \leq N$.
\end{theorem}

\begin{proofsketch}
The standard scenario approach in one dimension considers a set of i.i.d.\ performance samples $J_1, \dots, J_N$ and provides a bound on the probability that the next sampled performance is lower than the minimum. In our case, we only have lower bounds on the actual performances, with $\mathbb{P}\{J^{\gamma}_i \leq J_i\} \geq 1 - \gamma$. Assuming all lower bounds are valid with probability $(1-\gamma)^N$, we obtain an under-approximation of the true solution to the scenario program, and the union bound combines the confidences of the scenario approach and the validity of the lower bounds.
However, this confidence becomes very small for large sample sizes $N$, even when $1-\gamma$ is close to $1$. Conversely, with small sample sizes, the overall confidence remains low due to the weaker scenario confidence. To mitigate this, we require only $K$ of the $N$ lower bounds to be valid, which holds with high probability for small values of $\gamma$, even when $K$ is close to $N$. As a result, we can exclude a small number $N-K$ of samples, assuming them to be violated under the scenario approach, thereby soundly over-approximating the risk. This only marginally increases the stated risk bound while significantly reducing the confidence overhead. Since we cannot specify which bounds are valid, we use the minimum over all lower performance bounds as an under-approximation of the solution to all scenario sub-programs with $N-K$ discarded constraints, providing a sound performance guarantee.
The complete proof, including detailed bounds and derived inequalities,
can be found in
\ifthenelse{\isundefined{\techreport}}{\cite{arxiv}}{Appendix~\ref{sec:appB}}.
\end{proofsketch}

Theorem~\ref{thm:bound} bounds the risk for a policy to achieve a performance less than the minimum performance on any of the IMDPs. This bound only depends on values we can observe from the learned IMDP approximations. The theorem includes a tuning parameter $K \leq N$. The bound is valid for any value of $K$, and to obtain the tightest bound, we enumerate possible values and solve the resulting equation. For a fixed $K$, the left-hand side of Equation~\eqref{eqn:binom} is constant, and the right-hand side is the cumulative distribution function of a beta distribution with $K + 1$ and $N - K$ degrees of freedom~\citep{campi2018introduction}, which is easy to solve numerically for its unique solution in the interval $[0,1]$ using bisection~\cite{DBLP:journals/sttt/BadingsCJJKT22,ross2014introduction}. To the best of our knowledge, this is the first result to establish PAC guarantees on policy performance in unseen environments of upMDPs, in a setting where sample environments are unknown and can only be estimated from trajectories. Figure~\ref{fig:risknormalfull} illustrates the resulting risk bounds for example values. We assess the quality and tightness of our performance and risk bounds in the benchmarks presented in Section~\ref{sec:exp}.
% \begin{figure}[t]
%     \centering
%     \begin{minipage}[t]{0.48\textwidth}
%         \centering
%         \includegraphics[width=0.8\textwidth]{Figures/risks/fig_risk_gamma=0.0001.pdf}
%         \caption{Tightest risk bounds $\varepsilon$ from Theorem~\ref{thm:bound} obtained for $N$ verification samples. With IMDP confidence $\gamma = 10^{-4}$, against common values for $\eta$.}
%         \label{fig:riskcurves}
%     \end{minipage}
%     \hfill
%     \begin{minipage}[t]{0.48\textwidth}
%         \centering
%         \includegraphics[width=0.78\textwidth]{Figures/diffs_betting_200.pdf}
%         \caption{Verification set of size $N = 200$ with performance guarantees $\tilde{J} = 30.4$ with risk bound $\varepsilon = 0.04$ and empirical risk $\hat{\varepsilon} = 0.006$, and $\tilde{J}_7 = 46.7$ with $\varepsilon = 0.09$ and $\hat{\varepsilon} = 0.04$.}
%         \label{fig:riskdiscard}
%     \end{minipage}
% \end{figure}
% Figure~\ref{fig:riskcurves} shows the obtained risk bounds for various combinations of confidences and verification set sizes. Appendix~\ref{sec:apprisk} extends this with an empirical analysis of the bound tightness. 

We extend Theorem~\ref{thm:bound} to allow tuning of the risk-performance trade-off by \textit{discarding} samples~\citep{DBLP:journals/jota/CampiG11}. Instead of bounding the risk for the policy to achieve a performance less than the minimum, we state a bound for the $k$th order statistic of the verification set. Users can choose $k$ for a permissible risk level and a potentially higher performance guarantee, avoiding constraints from samples in the unlikely tail of the distribution.

\begin{definition}[Order Statistic]
    For a set of $N$ samples $J_1,\dots,J_N \in \mathbb{R}$ and $0 \leq k < N$, the $k$th order statistic $\tilde{J}_{(k)}$ is the $k$th smallest element when arranging all samples from smallest to largest. \demo
\end{definition}

\begin{theorem}
    \label{thm:bounddiscard}
    Given $N$ i.i.d.\ sample MDPs $\mathcal{M}[\theta_i]$ and IMDPs $\mathcal{M}^\gamma[\theta_i]$, such that $\mathbb{P}\{\mathcal{M}[\theta_i] \in \mathcal{M}^\gamma[\theta_i]\} \geq 1 - \gamma$, for any policy $\pi$, confidence level $1 - \eta$ with $\eta > 0$, and number $k$ of discarded samples, it holds that
    \begin{equation}
        \mathbb{P}^N\left\{ r(\pi, \tilde{J}^{\gamma}_{(k)}) \leq \varepsilon_{(k)}(N, \gamma, \eta) \right\} \geq 1 - \eta,
    \end{equation}
    where $\tilde{J}^{\gamma}_{(k)}$ is the $k$th order statistic of the performances $J(\pi, \mathcal{M}^\gamma[\theta_i])$, and $\varepsilon_{(k)}(N, \gamma, \eta)$ is the solution to the following, for any $K \leq N - k$:
    \begin{equation}
        \sum_{i = K}^{N-k} \binom{N-k}{i}(1-\gamma)^i\gamma^{N-k-i} - (1 - \eta) = \sum_{i = 0}^{N-K} \binom{N}{i}\varepsilon^i(1-\varepsilon)^{N-i}.
        \label{eqn:binomdiscard} 
    \end{equation}
\end{theorem}
\vspace{-12pt}
\qed
\vspace{4pt}

When $k = 0$ and no samples are discarded, Theorem~\ref{thm:bounddiscard} specialises to Theorem~\ref{thm:bound}. The proof is an extension incorporating the additional uncertainty of the lower bounds obtained from the PAC IMDPs into the \emph{sampling-and-discarding} theorem from scenario optimisation~\cite{DBLP:journals/jota/CampiG11}. We detail the derivation of the bound in
\ifthenelse{\isundefined{\techreport}}{\cite{arxiv}}{Appendix~\ref{sec:appB}}
and analyse its tightness in the experiments in Section~\ref{sec:exp}.

\vspace*{-0.8em}
\subsection{Optimisations and Extensions}
\label{sec:opt}
\vspace*{-0.4em}
Finally in this section, we present some optimisations and extensions for our approach.
First we show that, if there is additional knowledge as to the parametric structure of the upMDP,
we can leverage this to obtain tighter approximations of the sample environments.
Conversely, we describe how to seamlessly apply our framework and results %for uncertainty quantification
to setups with \emph{less} model knowledge, i.e., where not even the graph structure nor the (possibly infinite) state space is known. Furthermore, we outline how our results apply to more general setups where 
parameters influence not only transition probabilities, but also the evaluation functions, i.e., the specifications or \emph{tasks} may vary across samples, aligning it with the setup commonly considered in meta reinforcement learning~\cite{DBLP:conf/nips/CollinsMS20,DBLP:conf/nips/GreenbergMCM23}.

\startpara{Model-based Optimisations}
IMDP learning as described in Section~\ref{sec:paclearning} requires no knowledge of an MDP beyond its graph structure. However, additional information about the environment can yield tighter approximations with fewer samples. In cases where certain parameters, like temperature or air pressure, and their effect on some transition probabilities are known exactly, those transitions can be treated as singleton intervals. This reduces the need for approximation and decreases the number of learned transitions $n_u$.
 
Additionally, we can apply \emph{parameter tying}~\citep{DBLP:conf/qest/PolgreenWHA16,DBLP:conf/icml/PoupartVHR06} to parameters appearing across different transitions. For instance, consider two transitions sharing the same parameterisation, $P_{\Theta}(s,a,s') = P_{\Theta}(t,b,t')$. We can combine the counts from both transitions since they represent the same Bernoulli experiment. Let $sim(s,a,s') = \{(t,b,t') \mid P_{\Theta}(s,a,s') = P_{\Theta}(t,b,t')\}$ denote the set of transitions with identical parametrisation. By plugging the combined counts, 
$\#^T(s,a,s') = \sum_{(t,b,t') \in sim(s,a,s')} \#(t,b,t')$ and $H^T(s,a,s') = \sum_{(t,b,t') \in sim(s,a,s')} \#(t,b)$ into Equations~\eqref{eqn:pointest} and~\eqref{eqn:intervalest}, we can obtain a tighter interval for both transitions.

Our experiments in Section~\ref{sec:exp} use model-based optimisations and the full evaluation in
\ifthenelse{\isundefined{\techreport}}{\cite{arxiv}}{Appendix~\ref{sec:appeval}}
compares the results with and without optimisations.
 
\startpara{Statistical Model Checking}
To approximate the performance of a learned policy in unknown sample environments, our framework is not limited to PAC IMDP learning. Various forms of statistical model checking (SMC)~\cite{DBLP:journals/tomacs/AghaP18,DBLP:conf/vmcai/HeraultLMP04,DBLP:conf/rv/LegayDB10} can be applied, as long as they provide a lower bound $J^{\gamma}_i$ on a policy’s performance in a single environment induced by parameters $\theta_i$, with a formally quantified confidence $\Pr\left\{J(\pi, \theta_i) \geq J^{\gamma}_i \right\} \geq 1 - \gamma$. SMC techniques that require less information than PAC IMDP learning include those that rely on the minimum probability $p_{min}$ potentially present in the model to infer the MDP’s end-components with the desired confidence~\cite{DBLP:conf/cav/AshokKW19,DBLP:conf/tacas/DacaHKP16,DBLP:journals/corr/meggendorfer24}, or those that operate in a fully black-box setting with no model knowledge, approximating only the performance value. However, the latter techniques are typically restricted to finite-horizon properties~\cite{DBLP:conf/cav/SenVA04,DBLP:conf/cav/YounesS02}. 

% Note that the presented results are not restricted to the probabilistic lower bounds on the policy's performance being obtained via PAC IMDP learning. Any SMC or statistical technique that produces probabilistic lower performance bounds  for each sample valuation $\theta_i$ of the form  can be applied. The performance and risk bounds are derived based on these lower bounds, accounting for the additional uncertainty they introduce. \ys{again decide on what to do with this}

\startpara{Uncertain Specifications or Tasks}
The meta reinforcement learning literature usually considers upMDPs where both transition probabilities and reward structure depend on parameters~\cite{DBLP:conf/icml/FinnAL17,DBLP:conf/nips/GreenbergMCM23,DBLP:conf/nips/GuptaMLAL18}. 
%We identified that 
Our framework encompasses this problem class, and our results carry across directly. Parameterised rewards or specifications can be integrated into the evaluation function $J$, allowing parameters to affect both transitions and rewards. While the formal methods community has explored uncertain rewards and specifications to a lesser extent~\cite{DBLP:conf/qest/0001HL19,DBLP:conf/valuetools/Scheftelowitsch17,DBLP:journals/iiset/SteimleKD21}, we believe this is a promising direction for future work, particularly in extending PAC guarantees to uncertain specifications and objectives. 

\startpara{Non-memoryless Policies}
In this work, we focus on synthesis of \emph{memoryless} policies, which are sufficient for optimal performance under many common performance functions, such as reachability and reach-avoid specifications, in both single MDPs and IMDPs. However, for multiple environments or more complex specifications, such as general linear-time (LTL) properties, an optimal robust policy may require memory~\cite{DBLP:journals/ior/NilimG05,DBLP:conf/fsttcs/RaskinS14,DBLP:conf/cdc/WolffTM12}. Our theoretical results in Theorems~\ref{thm:bound} and~\ref{thm:bounddiscard} apply to arbitrary policy classes, provided that policy performance can be evaluated on the learned IMDP approximations to obtain sample performance values.
But solution of IMDPs in these cases presents challenges.
For example, the standard automaton product constructions can be used to find finite-memory policies that optimise LTL specifications~\cite{DBLP:conf/cdc/WolffTM12} but, in addition to increased computational cost, the addition of memory means moving from a static to a dynamic uncertainty model~\cite{DBLP:conf/birthday/SuilenBB0025,DBLP:conf/cdc/WolffTM12}, yielding overly conservative performance bounds.

\vspace*{-0.7em}
\section{Experimental Evaluation}
\vspace*{-0.5em}
\label{sec:exp}
We implemented our framework as an extension of the PRISM model checker~\cite{DBLP:conf/cav/KwiatkowskaNP11}, which provides trajectory generation and (robust) value iteration for MDPs and IMDPs\footnote{The implementation is available at: \url{https://doi.org/10.5281/zenodo.14717176}.}.
We have evaluated our approach on a range of established benchmark environments used in similar work:
an Aircraft Collision Avoidance~\citep{kochenderfer2015}, a Chain Problem~\citep{DBLP:conf/ewrl/Araya-LopezBTC11}, a Betting Game~\citep{DBLP:journals/mmor/BauerleO11}, a Semi-Autonomous Vehicle~\citep{DBLP:conf/tacas/Junges0DTK16,Stckler2015NimbRoES}, the Firewire protocol~\cite{DBLP:conf/tacas/HartmannsKPQR19}, and the previously mentioned UAV~\citep{DBLP:journals/sttt/BadingsCJJKT22}. Table~\ref{tab:stats} shows the salient features of  the environments. We provide detailed descriptions of each benchmark, including the parameters and their distributions $\mathbb{P}$ in 
\ifthenelse{\isundefined{\techreport}}{\cite{arxiv}}{Appendix~\ref{sec:casestudies}}.

\begin{table}[t]
\centering
\caption{Salient characteristics of the evaluated benchmarks.} % Caption moved above the table
\resizebox{0.93\textwidth}{!}{%
\begin{tabular}{>{\centering\arraybackslash}m{3.5cm} >{\centering\arraybackslash}m{2.5cm} >{\centering\arraybackslash}m{1.2cm} >{\centering\arraybackslash}m{2.2cm} >{\centering\arraybackslash}m{2.0cm} >{\centering\arraybackslash}m{2.0cm}}
\toprule
\textbf{Benchmark} & \textbf{Evaluation $J$} & \textbf{Opt.} & \textbf{\#Parameters} & \textbf{\#States} & \textbf{\#Transitions} \\
\midrule
UAV~\cite{DBLP:journals/sttt/BadingsCJJKT22} & ${\bm\Pr}(\neg C \until T)$ & max& 15 & 4096 & 86912 \\
Aircraft Collision~\citep{kochenderfer2015} & ${\bm\Pr}(\neg C \until T)$ & max & 2 & 303 & 3468  \\
Firewire~\cite{DBLP:conf/tacas/HartmannsKPQR19} & ${\bm\Pr}(\lozenge T)$ & min &1 & 80980 & 112990  \\
Semi-Auton.\ Vehicle~\citep{DBLP:conf/tacas/Junges0DTK16} & ${\bm\Pr}(\lozenge T)$ & max & 2 & 411 & 1503  \\
 Betting Game~\citep{DBLP:journals/mmor/BauerleO11} & $\mathbb{E}(\lozenge T)$ & max & 1 &  480 & 2730 \\
Chain~\citep{DBLP:conf/ewrl/Araya-LopezBTC11} & $\mathbb{E}(\lozenge T)$  & min & 1 & 7 & 42  \\
\bottomrule
\end{tabular}
} % end of resizebox
\label{tab:stats}
\vspace{-7pt}
\end{table}

\startpara{Setup}
Since our approach is the first to provide policy performance guarantees
under two layers of uncertainty, i.e., unknown sample environments from an unknown distribution,
our experiments focus on assessing the quality of these guarantees.
We study: (1) the tightness of the performance level $\tilde{J}$, assessing how closely our stated robust performance guarantee derived from the learned PAC IMDPs aligns with the actual robust performance on the true underlying sample MDPs that are hidden from the algorithm; (2) the quality of the risk bound $\varepsilon$ derived from Theorems~\ref{thm:bound} and~\ref{thm:bounddiscard} with respect to the true violation risk $r(\pi, \tilde{J})$.

Our approach includes two sampling dimensions corresponding to the two layers of uncertainty: (1) the number of unknown MDPs induced by parameter valuations sampled from the distribution $\mathbb{P}$; (2) the number of sample trajectories generated in each unknown MDP. For the first dimension we consider a total of $600$ sample MDPs, which we divide equally into training and verification sets. For the Firewire benchmark, we consider $150$ verification samples. The second dimension is evaluated for up to $10^6$ trajectories in each sampled environment.

%\ys{mention that we only use $N = 150$ verification samples for Firewire}

For policy learning, we consider the two methods described in Section~\ref{sec:policylearning}: robust IMDP policy learning and robust meta reinforcement learning with the max-min objective, implemented using the \emph{Gymnasium} Python framework~\cite{DBLP:journals/corr/gym}.
This illustrates the applicability of our approach to distinct state-of-the-art policy learning algorithms. We focus here on guarantees, rather than comparing policy learning methods, but we include statistics for both in
\ifthenelse{\isundefined{\techreport}}{\cite{arxiv}}{Appendix~\ref{sec:appeval}}
and refer the reader to, e.g.,~\cite{DBLP:conf/nips/CollinsMS20,DBLP:conf/nips/GreenbergMCM23,DBLP:conf/nips/SuilenS0022} for an in-depth comparison of methods.

%We apply PAC IMDP learning to the verification set, leveraging the PRISM IMDP engine to evaluate the learned robust policy on the resulting IMDPs.

For producing guarantees, we set the inclusion confidence level for the PAC IMDPs learned on the verification set to $\gamma = 10^{-4}$ and the overall confidence when applying Theorems~\ref{thm:bound} and~\ref{thm:bounddiscard} to $\eta = 10^{-2}$.
Optimisations from Section~\ref{sec:opt} are applied (see
\ifthenelse{\isundefined{\techreport}}{\cite{arxiv}}{Appendix~\ref{sec:appeval}}
for results of their impact).
All experiments were conducted on a 3.2 GHz Apple M1 Pro CPU with 10 cores and 16 GB of memory. 

\begin{table}[t]
\centering
\caption{Resulting performances, guarantees and risk bounds.} % Caption moved above the table
\renewcommand{\arraystretch}{1.1} % Increase row spacing
\resizebox{\textwidth}{!}{%
\begin{tabular}{>{\centering\arraybackslash}m{2.9cm} >{\centering\arraybackslash}m{2.1cm} >{\centering\arraybackslash}m{2cm} >{\centering\arraybackslash}m{2cm} >{\centering\arraybackslash}m{2.1cm} >{\centering\arraybackslash}m{2.2cm} >{\centering\arraybackslash}m{2.2cm} >{\centering\arraybackslash}m{2.0cm}}
\toprule
\small{\textbf{Benchmark}} & \small{\textbf{Performance $J$}} & \small{\textbf{Guarantee \newline \centering $\tilde{J}$}} & \small{\textbf{Risk Bound $\varepsilon$}} & \small{\textbf{\hspace*{1.5mm}Empirical \newline Risk $r(\pi, \tilde{J})$}} & \scriptsize{\textbf{Risk Bound  
$\varepsilon_{(5)}$/Empirical Risk}} & \scriptsize{\textbf{Risk Bound $\varepsilon_{(10)}$/Empirical Risk}} & \scriptsize{\textbf{\hspace*{0.7mm} Runtime\newline \hspace*{2mm} per $10^4$\newline trajectories}} \\
\midrule
UAV & 0.7110  & 0.7100 &  0.027 & 0.003 & 0.052 / 0.023 & 0.075 / 0.057 &  1.51s \\
Aircraft Collision & 0.5949 & 0.5900 & 0.027 & 0.004 & 0.052 / 0.017 & 0.075 / 0.046 & 0.35s \\
Firewire & 0.1946 & 0.1967 & 0.055 & 0.004 & 0.103 / 0.039 & 0.146 / 0.081 & 14.9s \\
Semi-Auton.\ Vehicle & 0.7854 & 0.7767 & 0.027 & 0.004 & 0.052 / 0.018 & 0.075 / 0.033 & 0.50s \\
Betting Game & 30.78  & 30.65 & 0.027 & 0.005 & 0.052 / 0.016 & 0.075 / 0.026 & 1.12s \\
Chain & 127.2 & 128.0 & 0.027 & 0.002 & 0.052 / 0.032 & 0.075 / 0.054 & 0.32s \\
\bottomrule
\end{tabular}
} % end of resizebox
\label{tab:results}
\vspace{-12pt}
\end{table}

\startpara{Results}
Table~\ref{tab:results} summarises the resulting performances and guarantees for the best-performing policy. All results are obtained after processing the full set of trajectories. We first give the \emph{true robust performance} \(J\), i.e., the policy's performance on the worst-case true MDP, which is hidden from the algorithm. We then report the key outputs of our approach: the \emph{robust performance guarantee}~\(\tilde{J}\), representing the worst-case performance on the learned PAC IMDPs, and the \emph{risk bound} \(\varepsilon\) for the performance guarantee \(\tilde{J}\), obtained using Theorem~\ref{thm:bound}. We also show an empirical estimate of the \emph{true violation risk} \(r(\pi, \tilde{J})\), computed over 1000 fresh sample MDPs. To evaluate the bounds derived via Theorem~\ref{thm:bounddiscard}, we include the risk bounds \(\varepsilon_{(k)}\) for discarding \(k = 5\) and \(k = 10\) worst-case samples, alongside estimates of the corresponding true violation risks.

For example, in the UAV case study, the table shows that the learned policy achieves at least \(J = 0.711\) probability of reaching the goal without crashing into an obstacle on any sampled true MDP hidden from the algorithm. The learned IMDP approximations certify a minimum performance of \(\tilde{J} = 0.71\), with the probability of performing worse on a fresh sample MDP bounded by \(\varepsilon = 0.027\). On 1000 fresh samples, the policy actually performed worse in only \(0.3\%\) of cases.

%  \emph{true robust performance} $J$, which refers to the performance of the learned policy on the worst-case true MDP in the verification set, unknown to the algorithm, along with the corresponding. 
% %for $\pi$ to achieve a performance less than $\tilde{J}$

Figure~\ref{subfig:perfplot} shows the learning process and the derived performance guarantee for the Betting Game benchmark. We plot the true robust performance $J$ (solid line), and the robust performance guarantee $\tilde{J}$ (dashed line) against the number of trajectories processed for each unknown MDP.
%, both in the training and the verification sets. 
We depict the progress for robust IMDP policy learning (purple) and robust meta reinforcement learning (yellow). The dotted red line corresponds to the \emph{existential guarantee}~\cite{DBLP:journals/sttt/BadingsCJJKT22}, i.e., the minimum performance on any MDP from the verification set, when applying the individual optimal policies,
%and not a single policy to all environments, 
which constitutes a natural upper bound on robust policy performance. Figure~\ref{subfig:riskplot} illustrates the risk-tuning with performance guarantees obtained via Theorem~\ref{thm:bounddiscard}. We depict the performances
%of the learned robust policy 
on the PAC IMDPs learned for the verification set (pink dots) and the performance guarantees $\tilde{J}_{(k)}$ when discarding the $k=0,5, \text{and }10$ worst-case samples (dashed lines), corresponding to the risk bounds $\varepsilon_{(k)}$ in Table~\ref{tab:results}. The full results for all policy learning techniques and benchmarks can be found in
\ifthenelse{\isundefined{\techreport}}{\cite{arxiv}}{Appendix~\ref{sec:appeval}}.

% We implement our approach 
% %for various IMDP learning methods used in robust policy learning,
% on top of the PRISM verification tool~\cite{DBLP:conf/cav/KwiatkowskaNP11}, which provides trajectory generation and (robust) value iteration for MDPs and IMDPs. We compare several IMDP policy learning approaches for the training set (see Section~\ref{sec:policylearning}). These include: (1) PAC learning (Section~\ref{sec:paclearning}), (2) Linearly Updating Intervals~\citep{DBLP:conf/nips/SuilenS0022}, (3) UCRL2 reinforcement learning~\citep{DBLP:conf/nips/AuerJO08}, and (4) Maximum a-posteriori (MAP) point estimates~\citep{DBLP:conf/nips/SuilenS0022}. We adapt a modified version of UCRL2 from~\citet{DBLP:conf/nips/SuilenS0022} for IMDPs. Details of each learning algorithm are provided in Appendix~\ref{sec:appA}.

% Our approach involves two sampling dimensions: sampling unknown MDPs and sampling trajectories within each unknown MDP. Increasing the amount of MDP samples reduces the risk for the performance guarantee. Increasing the number of sampled trajectories leads to tighter approximations, aligning the observable performance guarantee closer to the actual performance. Both dimensions are essentially independent. The risk follows from the results in Section~\ref{sec:riskeval} and is detailed in Appendix~\ref{sec:apprisk}. The following empirical analysis addresses the second dimension, and for each method we consider:

\begin{figure}[t]
    \centering
    \begin{minipage}[b]{0.47\textwidth} % Adjust the width as needed
        \centering
        \includegraphics[width=0.85\linewidth]{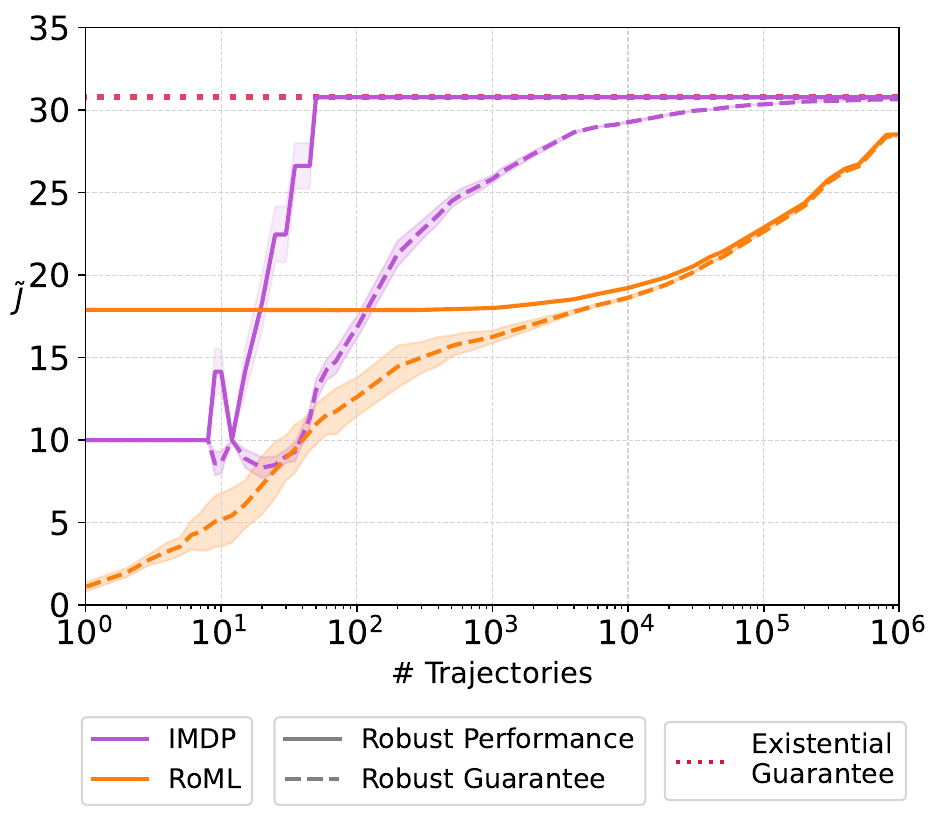}
        \caption{True robust performances $J$ and guarantees $\tilde{J}$ against  number of processed trajectories (betting game).}
        \label{subfig:perfplot}
    \end{minipage}
    \hfill
    \begin{minipage}[b]{0.47\textwidth} % Adjust the width as needed
        \centering
        \includegraphics[width=0.89\linewidth]{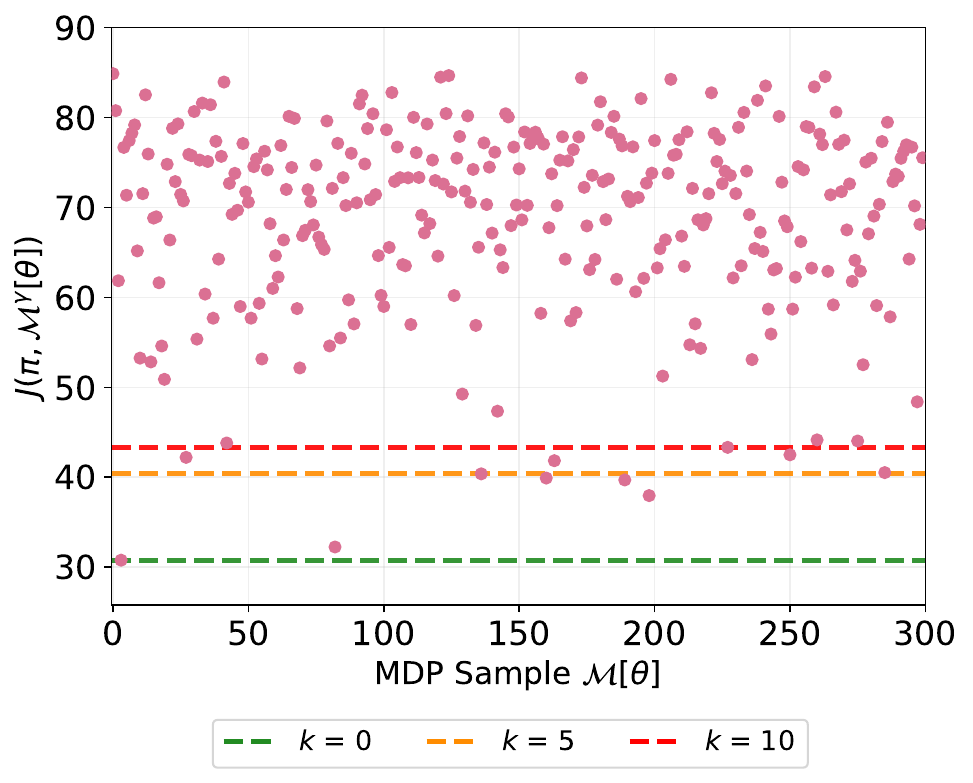}
        \vspace{3pt}
        \caption{Robust performance guarantees $\tilde{J}_{(k)}$, when discarding the $k=0$, $5$, and $10$ worst-case samples (betting game).}
        \label{subfig:riskplot}
    \end{minipage}
    \vspace{-12pt}
\end{figure}

\startpara{Discussion} 
The results show that our framework generates tight bounds on the performance and risk of learned policies in upMDPs. Our approach effectively addresses and integrates the two layers of uncertainty: (1) an unknown environment distribution and (2) unknown sample environments. 
%The framework successfully leverages PAC IMDP learning to soundly and precisely approximate the robust performance of learned policies across the unknown sample environments with high confidence. Furthermore, 
Our results yield high-quality risk bounds for the performance of policies in unseen environments. They enable tuning the risk-performance trade-off, and despite incorporating two layers of uncertainty, provide useful bounds with high user-specified confidence, constituting the first PAC guarantee for this general setup. While policy learning and solving PAC IMDPs scales with the model size and the number of sample MDPs, the computation of the risk bounds via Theorems~\ref{thm:bound} and~\ref{thm:bounddiscard} depends solely on the number of verification samples $N$ and is independent of the model size.
Regarding scalability, we briefly note that the range of model sizes we handle (see Table~\ref{tab:stats}) includes the largest instances handled by comparable methods that perform PAC IMDP learning from trajectories~\cite{DBLP:conf/cav/AshokKW19,DBLP:conf/nips/SuilenS0022,DBLP:journals/corr/meggendorfer24}.

\section{Conclusion}
\label{sec:conc}
We have presented a novel approach for producing certifiably robust policies for MDPs with epistemic uncertainty, where transition probabilities depend on parameters with unknown distributions. We have demonstrated
%on a range of relevant case studies
that our approach yields tight bounds on a policy's performance in unseen environments from the same distribution. Future work includes extending certifiably robust policies to settings where the specification or task is also uncertain and parameter-dependent.

\begin{credits}
\subsubsection{\ackname}
This work was part funded by
% Alessandro Abate acknowledges support from 
the ARIA projects SAINT and Super Martingale Certificates,
the UKRI AI Hub on Mathematical Foundations of AI
and the ERC under the European Union’s Horizon 2020 research and innovation programme (FUN2MODEL, grant agreement No.~834115).
\end{credits}

\bibliographystyle{splncs04}
\bibliography{references}

\ifthenelse{\isundefined{\techreport}}{
}{

\newpage
% \vspace{18pt}

\appendix
\section{Derivations and Proofs}
\label{sec:appB}
We detail the proofs and derivations of our main theoretical contributions. We state our lemmas and theorems for the case of maximising the evaluation functions, which has implications for the linear programs and inequalities used. However all our results apply to the minimising case by swapping the inequalities and directions of optimisation.

First, we present a range of results for incorporating additional uncertainty into the \textit{scenario approach}~\citep{DBLP:journals/siamjo/CampiG08,DBLP:journals/jota/CampiG11}. We then show how to formulate our problem as a randomised convex program—the \textit{scenario program}—and apply the generalised scenario theorems to derive our main results.

\subsection{The Scenario Approach with Uncertain Constraints}

We first present the basic setup for the scenario approach introduced in~\cite{DBLP:journals/siamjo/CampiG08}. The ingredients for the scenario approach are:
\begin{enumerate}
    \item A cost function $c^Tx$,
    \item An admissible region $C \subseteq \mathbb{R}^d$,
    \item A family of convex constraints indexed by an uncertain parameter $\{C_\theta \subseteq \mathbb{R}^d \mid \theta \in \Theta\}$,
    \item A probability measure $\mathbb{P}$ over $\Theta$.
\end{enumerate}

Given samples $(\theta_1, \dots , \theta_N)$ of independent random parameter valuations drawn from $(\Theta, \mathbb{P})$, a \textit{scenario program} is a linear program over the corresponding convex constraints:

\begin{equation}
\begin{array}{ll@{}ll}
\quad \quad \displaystyle\max_{x \in C}  & c^Tx &\\
\text{subject to} & \displaystyle x \in \bigcap_{i = 1,\dots,N} C_{\theta_i},&\\
\end{array}
\label{eqn:sp}
\end{equation} 
which we specialise to $c^Tx = x$. Let $x^*$ be the solution to the scenario program in Equation~\eqref{eqn:sp}. The scenario approach provides the generalisation theory to bound the \textit{violation probability} defined as:

\begin{equation}
    V(x) = \mathbb{P} \left\{\theta \in \Theta \colon x \not \in C_\theta\right\}.
\end{equation}

The violation probability is the probability that the solution to the scenario program is violating an unseen constraint sampled from $\Theta$, which corresponds to the policy violation risk $r(\pi,\theta)$. From now on, we assume the sampled constraints $C_{\theta_1}, \dots, C_{\theta_N}$ (abbr.~: $C_{\theta_{1:N}}$) are unknown. Instead, we are given a set of known convex constraints $\hat{C}_{\theta_{1:N}}$ , for which $\hat{C}_{\theta_i} \subseteq C_{\theta_i}, \forall 1 \leq i \leq N$. Furthermore, we assume that all constraints, known or unknown are one-dimensional intervals of the form $C_{\theta_i} = [a, b_{\theta_i}]$, for some constant $a \leq b_{\theta_i},\forall 1 \leq i \leq N$. Note that $a$ can be $-\infty$. All our results hold for the case of constraints as one-dimensional intervals, but do not necessarily generalise to higher dimensional problems. In the minimisation case, the intervals would be of the form $[a_{\theta_i}, b]$. 

We first show that the violation probability of the solution obtained for the more conservative constraints $\hat{C}_{\theta_i}$ cannot be higher than for the unknown constraints $C_{\theta_i}$. This means that a higher solution is associated with a higher risk, which reflects the risk-performance trade-off.
\begin{lemma}
    \label{lem:smallerRisk}
    Let $a \leq x \leq y \in \mathbb{R}$, it holds that $V(x) \leq V(y)$.
\end{lemma}
\begin{proof}
    Given a new constraint $C_\theta = [a, b_\theta]$, we have that

\begin{equation}
    \label{eqn:lemineq}
    \begin{aligned}
        x \not \in C_{\theta} \Rightarrow x > b_\theta \Rightarrow y > b_\theta \Rightarrow y \not \in C_{\theta}.
    \end{aligned}
\end{equation}
It follows that

\begin{equation}
    \begin{aligned}
        V(x) = \mathbb{P} \left\{ \theta \in \Theta \colon x \not \in C_\theta \right\} \overset{\text{\eqref{eqn:lemineq}}}{\leq} \mathbb{P} \left\{ \theta \in \Theta \colon y \not \in C_\theta \right\} = V(y).
    \end{aligned}
\end{equation}
\end{proof}

Next, we prove that soundly under-approximating the constraints also leads to a sound risk reduction. This will later allow us to bound the true risk of unknown constraints from known under-approximations.
\begin{lemma}
    \label{lem:allincluded}
    Let $x^* \in \mathbb{R}$ and $\hat{x}^* \in \mathbb{R}$ be the solutions to the scenario program in Equation~\eqref{eqn:sp} with constraints $\cons$ and $\consh$ with $\hat{C}_{\theta_i} \subseteq C_{\theta_i}, \forall 1 \leq i \leq N$. It holds that

    \begin{equation}
        V(\hat{x}^*) \leq V(x^*).
    \end{equation}

    % This does in general not hold for programs of dimension $n > 0$ or intervals with a variable lower bound.
\end{lemma}
\begin{proof}
    We show that the claim holds for our case of convex constraints of the form $C_{\theta_i} = [a, b_{\theta_i}]$. It is easy to see that the optimal solutions to the scenario program in Equation~\eqref{eqn:sp}, with finitely many interval constraints are

\begin{equation}
\begin{aligned}
    x^* = \min_{i} b_{\theta_i}\quad \text{and} \quad
    \hat{x}^* = \min_{i} \hat{b}_{\theta_i}.
\end{aligned}
\end{equation}
Now $\hat{C}_{\theta_i} \subseteq C_{\theta_i}$ is equivalent to $[a, \hat{b}_{\theta_i}] \subseteq [a, b_{\theta_i}]$, which implies that $\hat{b}_{\theta_i} \leq b_{\theta_i}$ and therefore $\hat{x}^* \leq x^*$. The claim follows by Lemma~\ref{lem:smallerRisk}.
% The counterexamples in Figure~\ref{sdf} show that the claim does in general not hold for higher dimensional constraints or intervals with variable lower bounds.
\end{proof}
Furthermore, we show that a solution obtained for constraints $\cons$ cannot have a higher violation probability than the solution for any of its subsets.
% \begin{tikzpicture}

% % Draw the larger circle and fill it with a lighter pastel color
% \fill[offwhite!50] (0,0) circle (2.5cm);
% \draw[very thick, alizarin!80] (0,0) circle (2.5cm);

% % Draw the smaller circle and fill it with a lighter pastel color
% \fill[offwhite] (-0.8,0.5) circle (1cm);
% \draw[very thick, alizarin!80] (-0.8,0.5) circle (1cm);

% % Draw and label the top point of both circles with larger dots
% \fill[red] (0,2.5) circle (3pt) node[above] {\large$\hat{x}^*$};
% \fill[red] (-0.8,1.5) circle (3pt) node[above] {\large $x^*$};

% % Add labels to each of the circles
% \node at (2.4, -2.2) {\large $\displaystyle \bigcap_i C_{\theta_i}$};
% \node at (0.4, -0.7) {\large $\displaystyle \bigcap_i \hat{C}_{\theta_i}$};

% \end{tikzpicture}

\begin{lemma}
    \label{lem:subset}
    Let $x^* \in \mathbb{R}$ and $x^*_Q \in \mathbb{R}$ be the solutions to the scenario program in Euqation~\eqref{eqn:sp} with constraints $\cons$ and $C_{\theta_{r_1:r_K}}$, for some $R = \{r_1, \dots, r_K\} \subseteq [N] = I$. It holds that 
    \begin{equation}
        V(x^*) \leq V(x^*_R).
    \end{equation}
\end{lemma}
\begin{proof}
    Given that all constraints are of the form $[a, b_{\theta_i}]$, it is easy to see that the optimal solutions to the scenario program in Equation~\eqref{eqn:sp}, with finitely many interval constraints are

\begin{equation}
\begin{aligned}
    x^* = \min_{i \in I} b_{\theta_i}\quad \text{and} \quad
    x^*_R = \min_{r \in R} b_{\theta_r}.
\end{aligned}
\end{equation}
Now since $R \subseteq I$ it follows that $x^* \leq x^*_R$. The claim follows by Lemma~\ref{lem:smallerRisk}.
\end{proof}

As the last ingredient for the first main theorem, we show that the violation risk for the solution obtained when removing $k$ arbitrary constraints cannot be higher than the solution obtained when removing the $k$ \textit{worst-case} constraints.

\begin{lemma}
    \label{lem:kremove}
    Let $x^*_R \in \mathbb{R}$ be the solution to the scenario program in Equation~\eqref{eqn:sp} with constraints $C_{\theta_{r_1:r_K}}$, for some $R = \{r_1, \dots, r_K\} \subseteq [N] = I$. Let $x^*_{N,k}$ be the solution for constraints $\cons$ that violates exactly $k = N - K$ of the $N$ constraints. It holds that
    \begin{equation}
        V(x^*_R) \leq V(x^*_{N,k}).
    \end{equation}
\end{lemma}

\begin{proof}
    Given that all constraints are of the form $[a, b_{\theta_i}]$, it is easy to see that the optimal solution to the scenario program in Equation~\eqref{eqn:sp}, with finitely many interval constraints $C_{\theta_{r_1:r_K}}$ is
    \begin{equation}
        x^*_R = \min_{r \in R} b_{\theta_r}.
    \end{equation}
    Further, the maximum solution that violates exactly $k$ of the $N$ constraints $\cons$ is the $(k+1)$-th smallest $b_\theta$. Since $|R| = K$ and $k = N - K$, it follows that $x^*_R \leq x^*_{N,k}$. The claim follows by Lemma~\ref{lem:smallerRisk}.
\end{proof}

We relax the requirement of all known constraints being sound under-approximations of the sampled, unknown constraints $\hat{C}_{\theta_i} \subseteq C_{\theta_i}, \forall 1 \leq i \leq N$, to an uncertain setting, where it holds that
\begin{equation}
    \mathbb{P}\left\{ \hat{C}_{\theta_i} \subseteq C_{\theta_i}\right\} \geq 1 - \gamma,
\end{equation}
for some $\gamma > 0$ and any $1 \leq i \leq N$. This represents the approximation of the performances up to a certain confidence via statistical model checking. The probability that there exists a subset of constraints with indices $R = \{r_1, \dots, r_K\} \subseteq [N]$, which all contain their under-approximations $\hat{C}_{\theta_{r}} \subseteq C_{\theta_{r}}, \forall r \in R$ is
\begin{align}
    & \mathbb{P}^N\left\{ \exists R \subseteq [N] \colon |R| = K \text{ and } \forall r \in R \colon \hat{C}_{\theta_{r}} \subseteq C_{\theta_{r}} \right\} \geq \sum_{i = K}^N \binom{N}{i}(1-\gamma)^i\gamma^{N-i}\label{eqn:firstineq}\\
    \overset{\eqref{lem:allincluded}}{\Rightarrow} & \mathbb{P}^N\left\{ \exists R \subseteq [N] \colon |R| = K \text{ and } V(\hat{x}^*_R) \leq V(x^*_R) \right\} \geq \sum_{i = K}^N \binom{N}{i}(1-\gamma)^i\gamma^{N-i}\\
    \overset{\eqref{lem:subset}}{\Rightarrow} & \mathbb{P}^N\left\{ \exists R \subseteq [N] \colon |R| = K \text{ and } V(\hat{x}^*) \leq V(x^*_R) \right\} \geq \sum_{i = K}^N \binom{N}{i}(1-\gamma)^i\gamma^{N-i}\\
    \overset{\eqref{lem:kremove}}{\Rightarrow} & \mathbb{P}^N\left\{ V(\hat{x}^*) \leq V(x^*_{N,k}) \right\} \geq \sum_{i = K}^N \binom{N}{i}(1-\gamma)^i\gamma^{N-i}, \label{eqn:finalineq}
\end{align}
where $k = N - K$, $\hat{x}^*$ denotes the solution to the scenario program in Equation~\eqref{eqn:sp} for all $N$ constraints (under-approximations), $x^*_{N,k}$ is the solution when removing the $k$ worst-case constraints, and $x^*_R$ and $\hat{x}^*_R$ are the solutions for the subset of constraints with indices $R$. Note that these results hold for any $K \leq N$. In the following, we abbreviate Equation~\eqref{eqn:finalineq} as
\begin{equation}
    \label{eqn:eqbound1}
    \mathbb{P}^N\left\{ V(\hat{x}^*) \leq V(x^*_{N,k}) \right\} \geq 1 - p.
\end{equation}

Given that $x^*_{N,k}$ is the solution to the scenario program in Equation~\eqref{eqn:sp} with constraints $\cons$, violating exactly $k$ of the $N$ constraints, we apply Theorem 2.1 of \cite{DBLP:journals/jota/CampiG11}, which states:
\begin{equation}
    \label{eqn:eqbound2}
    \mathbb{P}^N\left\{ V(x^*_{N,k}) \leq \varepsilon(N,k,\beta) \right\} \geq 1 - \beta,
\end{equation}
where $\varepsilon(N,k,\beta)$ is a bound on the violation probability given the number of discarded constraints $k$ and confidence level $1 - \beta$ with $\beta > 0$, given as the solution to
\begin{equation}
    \label{eqn:betaeps}
    \beta = \sum_{i = 0}^k \binom{N}{i}\varepsilon^i(1-\varepsilon)^{N-i}.
\end{equation}

To bound the violation probability of the observable solution $\hat{x}^*$ with respect to the distribution over unknown constraints $\mathbb{P}$, we use:
\begin{equation}
    V(\hat{x}^*) \leq V(x^*_{N,k}) \land V(x^*_{N,k}) \leq \varepsilon \Rightarrow V(\hat{x}^*) \leq \varepsilon.
\end{equation}
Thus,
\begin{equation}
    \mathbb{P}^N\left\{V(\hat{x}^*) \leq \varepsilon \right\} \geq \mathbb{P}^N\left\{ V(\hat{x}^*) \leq V(x^*_{N,k}) \land V(x^*_{N,k}) \leq \varepsilon \right\}.
\end{equation}
Using the union bound, this transforms into:
\begin{equation}
    \mathbb{P}^N\left\{V(\hat{x}^*) \leq \varepsilon \right\} \geq 1 - \mathbb{P}^N\left\{ V(\hat{x}^*) > V(x^*_{N,k}) \right\} - \mathbb{P}^N \left\{ V(x^*_{N,k}) > \varepsilon \right\}.
\end{equation}
From Equations~\eqref{eqn:eqbound1} and~\eqref{eqn:eqbound2}, we conclude:
\begin{equation}
    \label{eqn:bound}
    \mathbb{P}^N\left\{V(\hat{x}^*) \leq \varepsilon(N,k,\beta) \right\} \geq 1 - (\beta + p).
\end{equation}
We combined the uncertainties stemming from finite sample generalisation using the scenario approach, and the fact that constraints are only known as under-approximations up to a certain confidence. Equation~\eqref{eqn:bound} provides a bound on the violation probability for the solution obtained from known under-approximations $J(\hat{x}^*)$, generalising to the unknown distribution over true constraints from $\mathbb{P}$.

\begin{theorem}
    \label{thm:boundsp}
    Given $N$ i.i.d. samples $\theta_{1:N}\sim \mathbb{P}$ with corresponding unknown constraints $\cons$ and observable constraints $\consh$, all of the form $[a, b_\theta]$, and $\mathbb{P}\{\hat{C}_{\theta_i} \subseteq C_{\theta_i}\} \geq 1 - \gamma$. Let $\hat{x}^*$ be the solution to the scenario program in Equation~\eqref{eqn:sp}. Then for any $K \leq N$ and $\beta > 0$, it holds that
    \begin{equation}
        \mathbb{P}^N\left\{V(\hat{x}^*) \leq \varepsilon(N,k,\beta) \right\} \geq 1 - (\beta + p),
    \end{equation}
    with $k = N - K$ and $p = 1 - \sum_{i = K}^N \binom{N}{i}(1-\gamma)^i\gamma^{N-i}$, i.e.,
    \begin{equation}
        \mathbb{P}^N\left\{V(\hat{x}^*) \leq \varepsilon(N,k,\beta) \right\} \geq \sum_{i = K}^N \binom{N}{i}(1-\gamma)^i\gamma^{N-i} - \beta.
    \end{equation}
\end{theorem}

\begin{proof}
    The theorem follows directly from Equation~\eqref{eqn:bound} and the reasoning above.
\end{proof}

\subsection{Derivation of Theorem~\ref{thm:bound}}

We show that Theorem~\ref{thm:bound} follows as a special case of Theorem~\ref{thm:boundsp}. Consider an upMDP $\upmdp$ and a policy $\pi$. Given a evaluation function $J$, sampled parameters $\theta_i \sim \mathbb{P}$ induce convex constraints of the form:
\[
    C_{\theta_i} = (-\infty, J(\pi, \mathcal{M}[\theta_i])].
\]
Similarly, learned IMDP over-approximations $\mathcal{M}^\gamma[\theta_i]$, where $\mathcal{M}[\theta_i] \subseteq \mathcal{M}^\gamma[\theta_i]$, imply that $J(\pi, \mathcal{M}^\gamma[\theta_i]) \leq J(\pi, \mathcal{M}[\theta_i])$, inducing under-approximations of the constraints:
\[
    \hat{C}_{\theta_i} = (-\infty, J(\pi, \mathcal{M}^\gamma[\theta_i])] \subseteq (-\infty, J(\pi, \mathcal{M}[\theta_i])].
\]
Since $\mathbb{P}\{\mathcal{M}[\theta_i] \subseteq \mathcal{M}^\gamma[\theta_i]\} \geq 1 - \gamma$ by construction of the IMDPs, Theorem~\ref{thm:boundsp} becomes applicable to the solution $\hat{x}^* = \tilde{J} = \min_i J(\pi, \mathcal{M}^\gamma[\theta_i])$.

\begin{theorem}[\ref{thm:bound}]
    Given $N$ i.i.d.\ sample MDPs $\mathcal{M}[\theta_i]$ and IMDPs $\mathcal{M}^\gamma[\theta_i]$, such that $\mathbb{P}\{\mathcal{M}[\theta_i] \subseteq \mathcal{M}^\gamma[\theta_i]\} \geq 1 - \gamma$. For any policy $\pi$ and confidence level $1 - \eta$, with $\eta > 0$, it holds that
    \begin{equation}
        \mathbb{P}^N\left\{ r(\pi, \tilde{J}) \leq \varepsilon(N, \gamma, \eta) \right\} \geq 1 - \eta,
    \end{equation}
    where $\tilde{J} = \min_i J(\pi, \mathcal{M}^\gamma[\theta_i])$, and $\varepsilon(N, \gamma, \eta)$ is the solution to 
    \begin{equation}
        \sum_{i = K}^N \binom{N}{i}(1-\gamma)^i\gamma^{N-i} - (1 - \eta) = \sum_{i = 0}^{N-K} \binom{N}{i}\varepsilon^i(1-\varepsilon)^{N-i},
    \end{equation}
    for any $K \leq N$.
\end{theorem}

\begin{proof}
    By applying Theorem~\ref{thm:boundsp}, we obtain 
    \begin{equation}
        \label{eqn:proofineq}
        \mathbb{P}^N\left\{r(\pi, \tilde{J}) \leq \varepsilon(N, k, \beta) \right\} \geq \sum_{i = K}^N \binom{N}{i}(1-\gamma)^i\gamma^{N-i} - \beta,
    \end{equation}
    with $k = N - K$, $K \leq N$.
    
    Equating the right-hand side to $1 - \eta$, we obtain the following range of permissible $\beta$:
    \begin{equation*}
        \beta \leq \sum_{i = K}^N \binom{N}{i}(1-\gamma)^i\gamma^{N-i} - (1 - \eta).
    \end{equation*}
    
    Since the risk $\varepsilon(N, k, \beta)$ for fixed $N$ and $k$ increases as $\beta$ decreases, the smallest risk is obtained for the largest possible $\beta$, which corresponds to the smallest possible confidence that adds up to the desired confidence $1 - \eta$. Therefore, substituting
    \begin{equation*}
        \beta = \sum_{i = K}^N \binom{N}{i}(1-\gamma)^i\gamma^{N-i} - (1 - \eta)
    \end{equation*}
    into Equations~\eqref{eqn:betaeps} and ~\eqref{eqn:proofineq} concludes the proof.
\end{proof}

\subsection{Derivation of Theorem~\ref{thm:bounddiscard}}

To incorporate sample discarding into the setup with uncertain constraints, we adapt the reasoning above to exclude a fixed number \(l\) of the \(N\) observable constraints \(\consh\). Let \(L \subseteq [N]\) with \(|L| = l\) be the indices of the discarded constraints. The probability that there exists a subset of constraints with indices \(R = \{r_1, \dots, r_K\} \subseteq [N] \setminus L\), which all contain their under-approximations \(\hat{C}_{\theta_r} \subseteq C_{\theta_r}, \forall r \in R\) is
\begin{equation}
    \label{eqn:ineqdiscard}
    \begin{aligned}
        \mathbb{P}^N\Big\{ \exists R \subseteq [N] \setminus L \colon |R| = K \text{ and } 
        \forall r \in R \colon \hat{C}_{\theta_r} 
        &\subseteq C_{\theta_r} \Big\} \\
        &\geq \sum_{i = K}^{N-l} \binom{N-l}{i}(1-\gamma)^i \gamma^{N-l-i}.
    \end{aligned}
\end{equation}
Analogous to Equation~\eqref{eqn:firstineq}, we transform Equation~\eqref{eqn:ineqdiscard} into
\begin{equation}
    \label{eqn:discardfinalineq}
    \mathbb{P}^N\left\{ V(\hat{x}^*_{N,l}) \leq V(x^*_{N,m})\right\} \geq \sum_{i = K}^{N-l} \binom{N-l}{i}(1-\gamma)^i\gamma^{N-l-i},
\end{equation}
where \(x^*_{N,l}\) is the solution for constraints \(\consh\) without the indices \(L\), and \(m = N - K \).

\begin{theorem}[\ref{thm:bounddiscard}]
    Given $N$ i.i.d.\ sample MDPs $\mathcal{M}[\theta_i]$ and IMDPs $\mathcal{M}^\gamma[\theta_i]$, such that $\mathbb{P}\{\mathcal{M}[\theta_i] \subseteq \mathcal{M}^\gamma[\theta_i]\} \geq 1 - \gamma$, for any policy $\pi$, confidence level $1 - \eta$ with $\eta > 0$, and number $k$ of discarded samples, it holds that
    \begin{equation}
        \mathbb{P}^N\left\{ r(\pi, \tilde{J}_{(k)}) \leq \varepsilon_{(k)}(N, \gamma, \eta, k) \right\} \geq 1 - \eta,
    \end{equation}
    where $\varepsilon_{(k)}(N, \gamma, \eta, k)$ is the solution to 
    \begin{equation}
        \sum_{i = K}^{N-k} \binom{N-k}{i}(1-\gamma)^i\gamma^{N-k-i} - (1 - \eta) = \sum_{i = 0}^{N-K} \binom{N}{i}\varepsilon^i(1-\varepsilon)^{N-i},
    \end{equation}
    for any $K \leq N - k$.
\end{theorem}
\begin{proof}
    From Equation~\eqref{eqn:discardfinalineq} it follows that 

        \begin{equation}
        \label{eqn:proofineq2}
        \mathbb{P}^N\left\{r(\pi, \tilde{J}_k) \leq \varepsilon(N, m, \beta) \right\} \geq \sum_{i = K}^{N-k} \binom{N-k}{i}(1-\gamma)^i\gamma^{N-k-i} - \beta,
    \end{equation}
    with $m = N - K$, $K \leq N - k$. Equating the right-hand side to $1 - \eta$ and substituting 

    \begin{equation}
        \beta = \sum_{i = K}^{N-k} \binom{N-k}{i}(1-\gamma)^i\gamma^{N-k-i} - (1 - \eta)
    \end{equation}
    into Equations~\eqref{eqn:betaeps} and~\eqref{eqn:proofineq2} concludes the proof.
    
\end{proof}

\section{Benchmark Environments}
\label{sec:casestudies}
We detail the benchmarks used in our experimental evaluation in Section~\ref{sec:exp}.

\paragraph{Autonomous Drone.} The autonomous drone benchmark, as described in Section~\ref{sec:intro}, is adapted from~\cite{DBLP:journals/sttt/BadingsCJJKT22}. A drone manoeuvre in a 3D environment, starting from the origin (see Figure~\ref{fig:drone}). It aims to reach a target zone while avoiding obstacles. There are 15 parameters \(p_i\), one for each \(x\)-coordinate, influencing the probabilities of the drone drifting off. The evaluation function is the probability of reaching the goal without crashing into an obstacle.

\paragraph{Betting Game.} The betting game is a reward maximisation benchmark introduced in~\cite{DBLP:journals/mmor/BauerleO11}. The player starts with 10 coins and can sequentially place \(n\) bets, risking either 0, 1, 2, 5, or 10 coins. With probability \(p\), the player wins and earns double the bet; with probability \(1-p\), the bet is lost. The goal is to maximise the number of coins after \(n\) bets. The evaluation function is the expected number of coins after \(n\) bets. We consider a version with \(n = 8\) rounds of betting.

\paragraph{Chain Problem.} The chain benchmark was introduced in~\cite{DBLP:conf/ewrl/Araya-LopezBTC11} and consists of a chain of 6 states with two actions: (1) progressing to the next state with probability \(p\) and falling back to the initial state with probability \(1-p\), and analogous with inverse probabilities. The evaluation function is the expected number of steps required to reach the last state.

\paragraph{FireWire.} The FireWire example is a standard probabilistic verification benchmark~\cite{DBLP:conf/tacas/HartmannsKPQR19} modelling the execution of a root contention protocol used within the FireWire standard. A parameter \(p\) represents the probability for randomisation between two competing nodes in order to break symmetry. The evaluation function is the minimum probability of successful completion.

\paragraph{Aircraft Collision Avoidance.} The aircraft collision avoidance environment is a simplified version of the rich set of models introduced in~\cite{kochenderfer2015}. We consider a \(10 \times 5\) grid where two aircraft, one controlled by our agent and one adversarial, fly towards each other. In each step, both pilots may choose to fly straight, up, or down, succeeding with probabilities \(p\) and \(q\), respectively. The goal is for the agent to reach the opposite end of the grid without colliding with the adversarial aircraft, which manoeuvres arbitrarily. The evaluation function is the probability of the agent reaching the goal zone without colliding.

\paragraph{Semi-Autonomous Vehicle.} The semi-autonomous vehicle benchmark, introduced in~\cite{Stckler2015NimbRoES} and formalised as a PRISM model in~\cite{DBLP:conf/tacas/Junges0DTK16}, models an explorer moving through a grid while communicating with a controller via two faulty channels. The probabilities of each channel losing a message depend on parameters \(p\) and \(q\), and the agent's current position. In each step, the agent can either communicate over a chosen channel for a limited number of tries or move in a desired direction. The agent can only move a certain number of steps without successful communication; otherwise, the task fails. The evaluation function is the probability of the agent reaching a goal zone without exceeding the maximum number of steps without communication. We consider a \(10 \times 5\) grid, a maximum of two communication trials, and only two allowed moves without successful communication.

\section{Extended Experimental Evaluation}
\label{sec:appeval}

\begin{table}[t]
\centering
\caption{Extended benchmark statistics.} % Caption moved above the table
\resizebox{\textwidth}{!}{%
\begin{tabular}{>{\centering\arraybackslash}m{3.5cm} >{\centering\arraybackslash}m{2.5cm}>{\centering\arraybackslash}m{1.2cm} >{\centering\arraybackslash}m{2.3cm} > {\centering\arraybackslash}m{2.5cm} >{\centering\arraybackslash}m{1.8cm} >{\centering\arraybackslash}m{2.0cm}}
\toprule
\textbf{Benchmark} & \textbf{Evaluation $J$} & Opt. & \textbf{\#Parameters} & \textbf{Distribution $\mathbb{P}$} & \textbf{\#States} & \textbf{\#Transitions} \\
\midrule
UAV~\cite{DBLP:journals/sttt/BadingsCJJKT22} & ${\bm\Pr}(\neg C \until T)$  & max& 15 & $p_i \sim Beta(2,10)$ & 4096 & 86912 \\
\arrayrulecolor{lightgray}
\midrule
Aircraft Collision~\citep{kochenderfer2015} & ${\bm\Pr}(\neg C \until T)$ &max& 2 & $p \sim Beta(10,2)$\newline$q \sim Beta(2,10)$ & 303 & 3468  \\
\midrule
Firewire~\citep{DBLP:conf/tacas/HartmannsKPQR19} & ${\bm\Pr}(\lozenge T)$ &min& 1& $p \sim Beta(5,5)$ & 80980 & 112990  \\
\midrule
Semi-Auton.\ Vehicle~\citep{DBLP:conf/tacas/Junges0DTK16} & ${\bm\Pr}(\lozenge T)$ &max& 2& $p \sim Uni(.75,.95)$\newline$q\sim Uni(.55,.85)$ & 411 & 1503  \\
\midrule
 Betting Game~\citep{DBLP:journals/mmor/BauerleO11} & $\mathbb{E}(\lozenge T)$ &max& 1& $p \sim Beta(20,2)$ &  480 & 2730 \\
 \midrule
Chain~\citep{DBLP:conf/ewrl/Araya-LopezBTC11} & $\mathbb{E}(\lozenge T)$  &min& 1 & $p \sim Beta(5,5)$& 7 & 42  \\
\arrayrulecolor{black}
\bottomrule
\end{tabular}
} % end of resizebox
\label{tab:statsext}
\vspace{-5pt}
\end{table}

We present the results for the experimental evaluation in Section~\ref{sec:exp}. Table~\ref{tab:statsext} contains extended characteristics for each benchmark, including the underlying parameter distributions $\mathbb{P}$, unknown to the algorithm. $Uni(a,b)$ is a uniform distribution over the interval $[a,b]$, and $Beta(\alpha,\beta)$ is a Beta distribution with parameters $\alpha$ and $\beta$. Table~\ref{tab:resultsext} presents the extended results for both policies learned via IMDP learning and robust meta reinforcement learning. For IMDP learning we present the results for the best performing interval learning algorithm. Figure~\ref{fig:resfull} presents the full learning progress for all interval learning algorithms, which we describe in detail in Appendix~\ref{sec:appimdp}. Figure~\ref{fig:resfull} also shows the resulting performances and guarantees without model-based optimisations and parameter tying. For robust meta RL, we used directly parameterised policies~\cite{DBLP:books/lib/SuttonB98}.

\begin{table}[t]
\centering
\caption{Extended results for performances, guarantees and risk bounds.} % Caption moved above the table
\renewcommand{\arraystretch}{1.2} % Increase row spacing
\resizebox{\textwidth}{!}{%
\begin{tabular}{>{\centering\arraybackslash}m{2.9cm} >{\centering\arraybackslash}m{2.1cm} >{\centering\arraybackslash}m{2.1cm} >{\centering\arraybackslash}m{2.1cm} >{\centering\arraybackslash}m{2.1cm} >{\centering\arraybackslash}m{2.5cm} >{\centering\arraybackslash}m{2.5cm} >{\centering\arraybackslash}m{2.5cm} >{\centering\arraybackslash}m{2.5cm}}
\toprule
\small{\textbf{Benchmark}} &\textbf{Policy $\pi$}& \small{\textbf{Performance $J$}} & \small{\textbf{Guarantee \newline \centering $\tilde{J}$}} & \small{\textbf{Risk Bound $\varepsilon$}} & \small{\textbf{Empirical Risk $r(\pi, \tilde{J})$}} & \scriptsize{\textbf{Risk Bound 
$\varepsilon_{(5)}$/ Empirical Risk}} & \scriptsize{\textbf{Risk Bound $\varepsilon_{(10)}$/ Empirical Risk}} & \scriptsize{\textbf{Runtime per  $10^4$ \newline trajectories}} \\
\midrule
UAV & IMDP & 0.7110  & 0.7100 &  0.027 & 0.003 & 0.052 / 0.023 & 0.075 / 0.057 &  1.51s \\
 &  RoML & 0.6711  & 0.6700 &  0.027 & 0.003 & 0.052 / 0.023 & 0.075 / 0.057 &  1.51s \\
 \arrayrulecolor{lightgray}
 \midrule
Aircraft Collision & IMDP &0.5949 & 0.5907 & 0.027 & 0.004 & 0.052 / 0.017 & 0.075 / 0.046 & 0.35s \\
 & RoML &0.5879 & 0.5830 & 0.027 & 0.006 & 0.052 / 0.016 & 0.075 / 0.047 & 0.35s \\
  \midrule
  Firewire & IMDP &0.1946 & 0.1967 & 0.055 & 0.004 & 0.103 / 0.039 & 0.146 / 0.081 & 14.9s \\
 & RoML & 0.5973 & 0.5984 & 0.055 & 0.004 & 0.103 / 0.033  & 0.146 / 0.060  & 14.9s \\
  \midrule
Semi-Auton.\ Vehicle & IMDP & 0.7854 & 0.7767 & 0.027 & 0.004 & 0.052 / 0.018 & 0.075 / 0.033 & 0.50s \\
 &RoML & 0.0002 & 0.0002 & 0.027 & 0.003 & 0.052 / 0.010 & 0.075 / 0.034 & 0.50s \\
  \midrule
Betting Game & IMDP & 30.78  & 30.65 & 0.027 & 0.005 & 0.052 / 0.016 & 0.075 / 0.026 & 1.12s \\
 &RoML& 28.51  & 28.39 & 0.027 & 0.005 & 0.052 / 0.016 & 0.075 / 0.026 & 1.12s \\
  \midrule
Chain &IMDP& 485.4 & 487.3 & 0.027 & 0.003 & 0.052 / 0.010 & 0.075 / 0.034 & 0.32s \\
 &RoML & 127.2 & 128.0 & 0.027 & 0.002 & 0.052 / 0.032 & 0.075 / 0.054 & 0.32s \\
  \arrayrulecolor{black}
\bottomrule
\end{tabular}
} % end of resizebox
\label{tab:resultsext}
\vspace{-5pt}
\end{table}

\begin{figure}[h]
    \centering

    \begin{subfigure}[t]{\textwidth}
    \centering
        \includegraphics[width=0.82\textwidth]{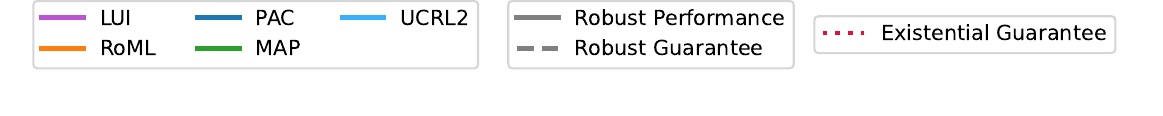}
    \end{subfigure}
    
    \begin{subfigure}[t]{\textwidth}
        \centering
        % Row 3, first subfigure
        \includegraphics[width=0.31\textwidth]{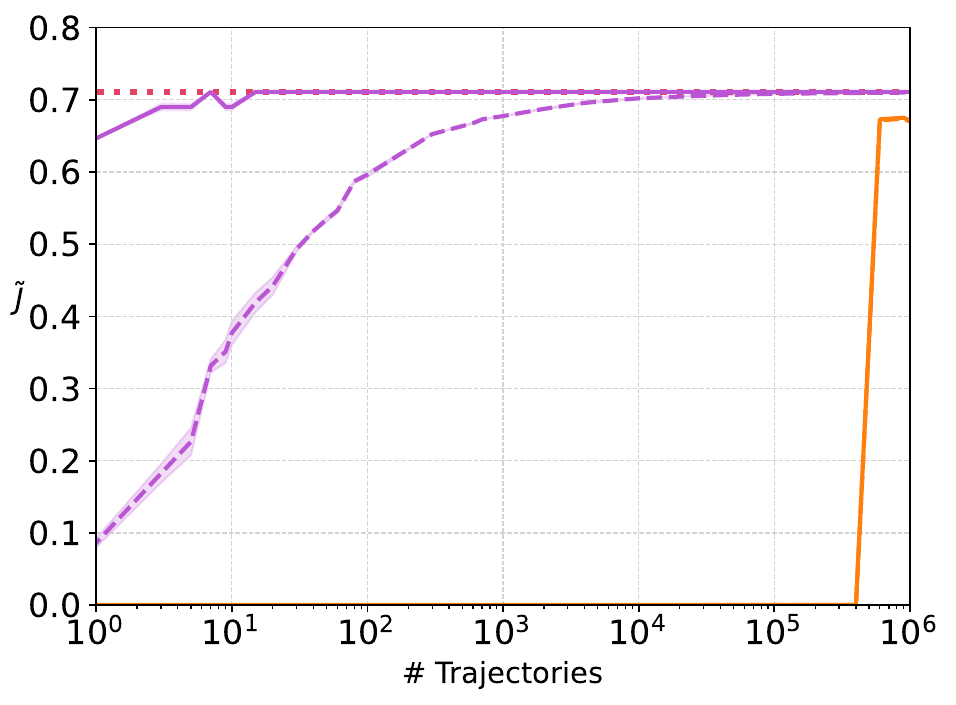}
        \hfill
        \includegraphics[width=0.31\textwidth]{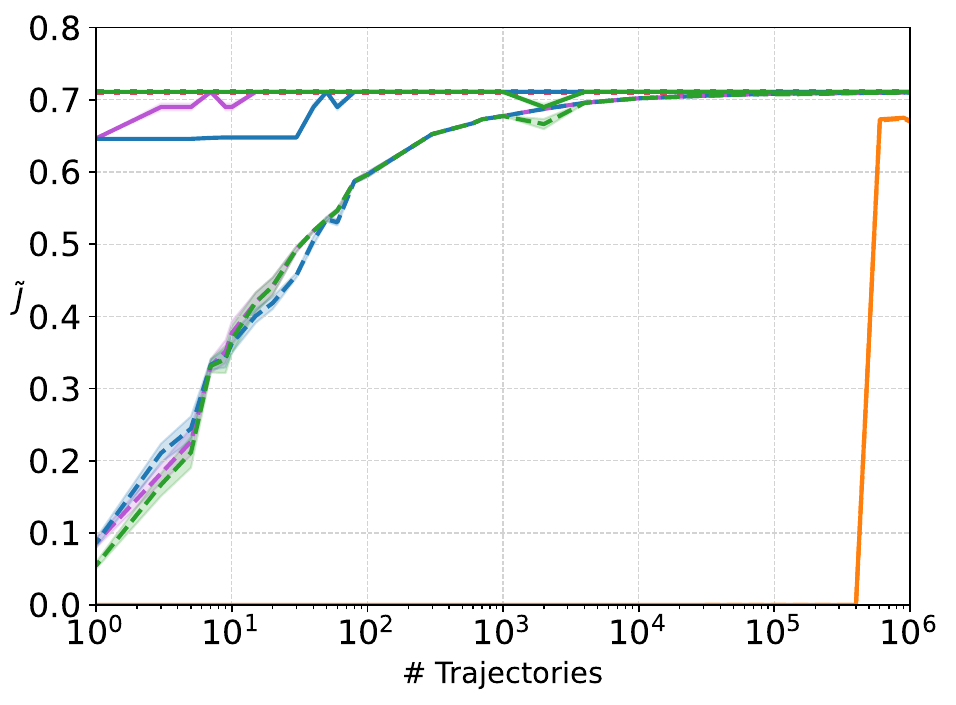}
        \hfill
        % Row 3, second subfigure
        \includegraphics[width=0.31\textwidth]{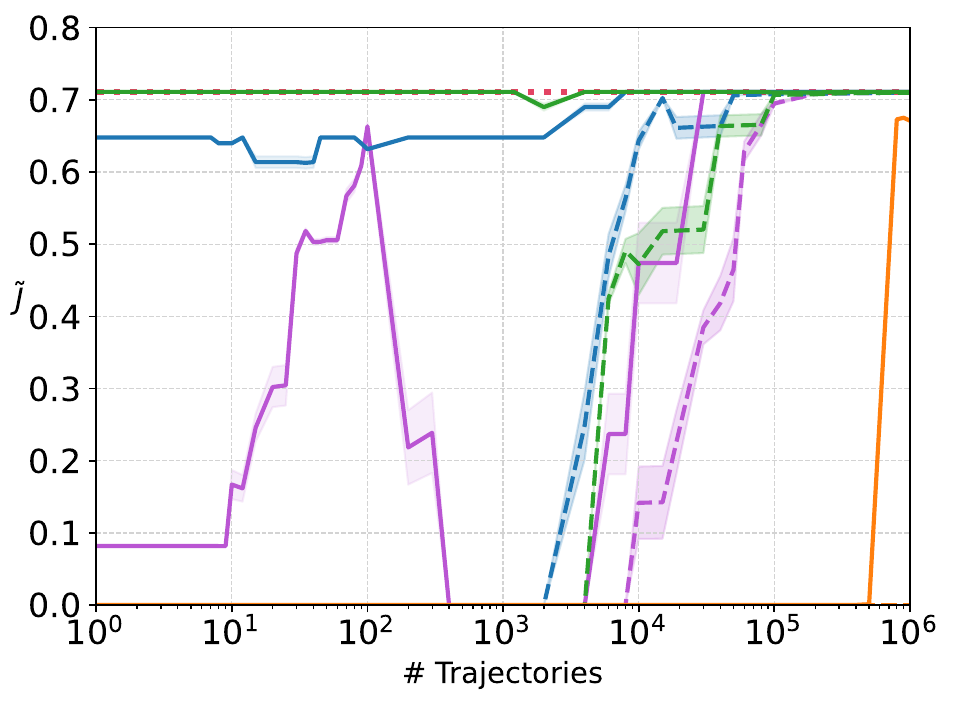}
        \caption{UAV.}
    \end{subfigure}

    \vspace{0.2em} % Add some vertical space between rows

    \begin{subfigure}[t]{\textwidth}
        \centering
        % Row 3, first subfigure
        \includegraphics[width=0.31\textwidth]{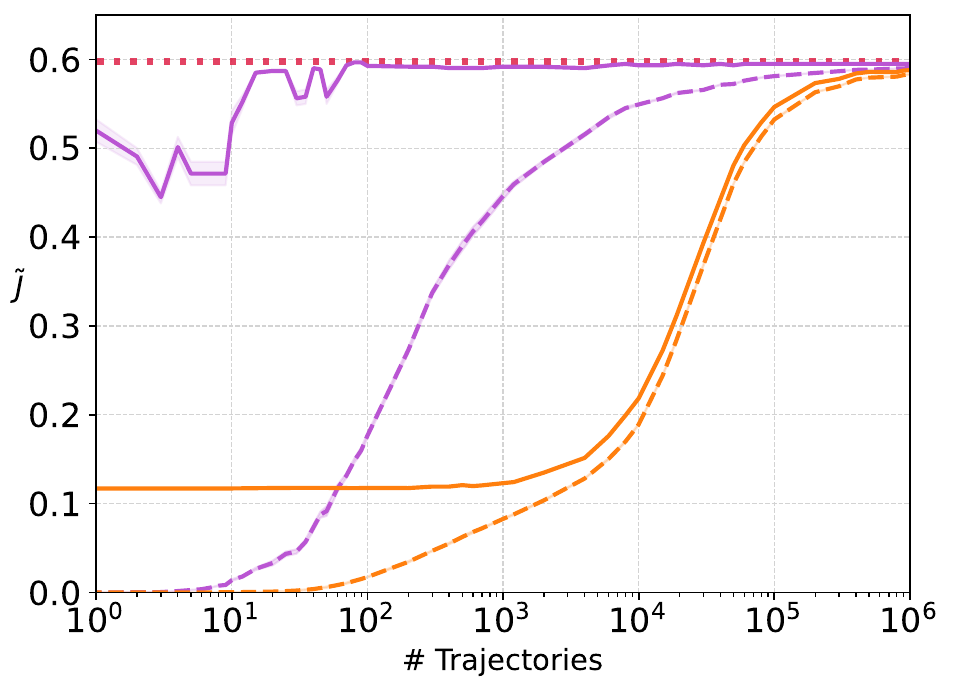}
        \hfill
        \includegraphics[width=0.31\textwidth]{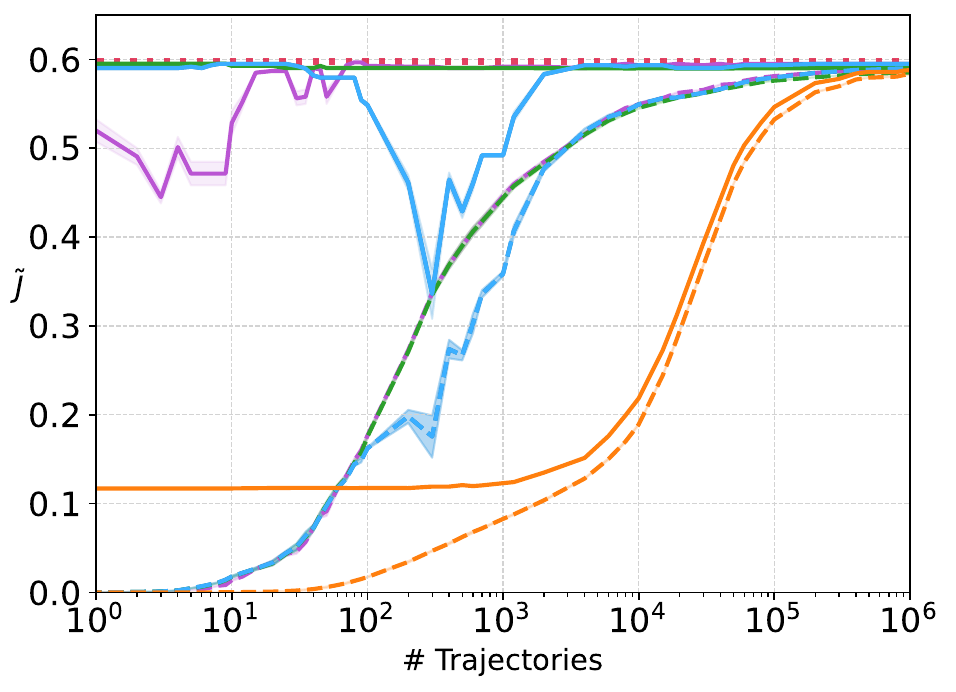}
        \hfill
        % Row 3, second subfigure
        \includegraphics[width=0.31\textwidth]{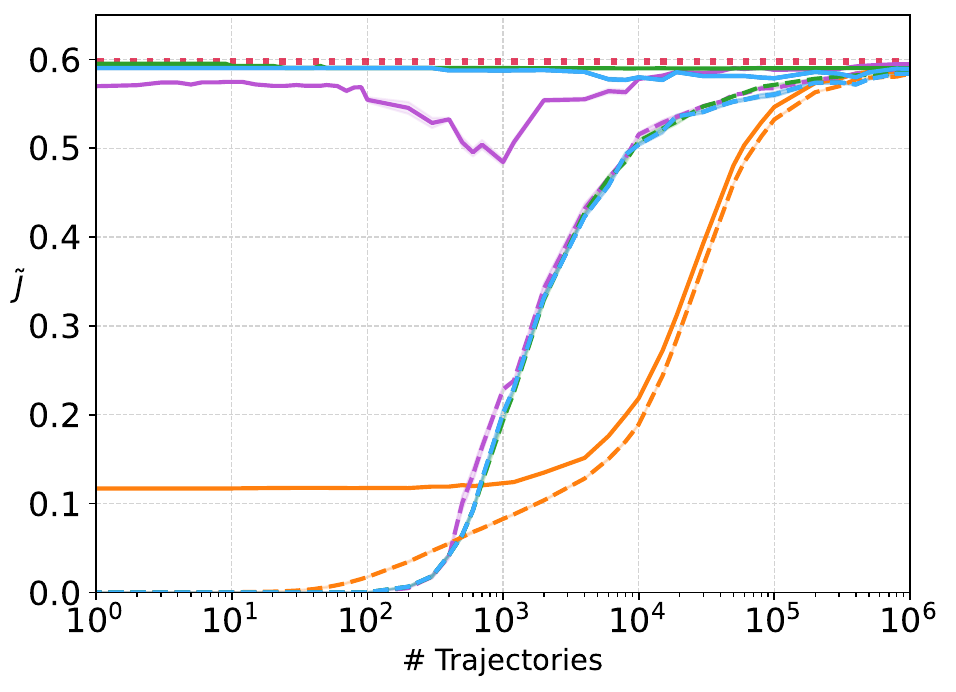}
        \caption{Aircraft Collision Avoidance.}
    \end{subfigure}

    \vspace{0.2em} % Add some vertical space between rows

    \begin{subfigure}[t]{\textwidth}
        \centering
        % Row 3, first subfigure
        \includegraphics[width=0.31\textwidth]{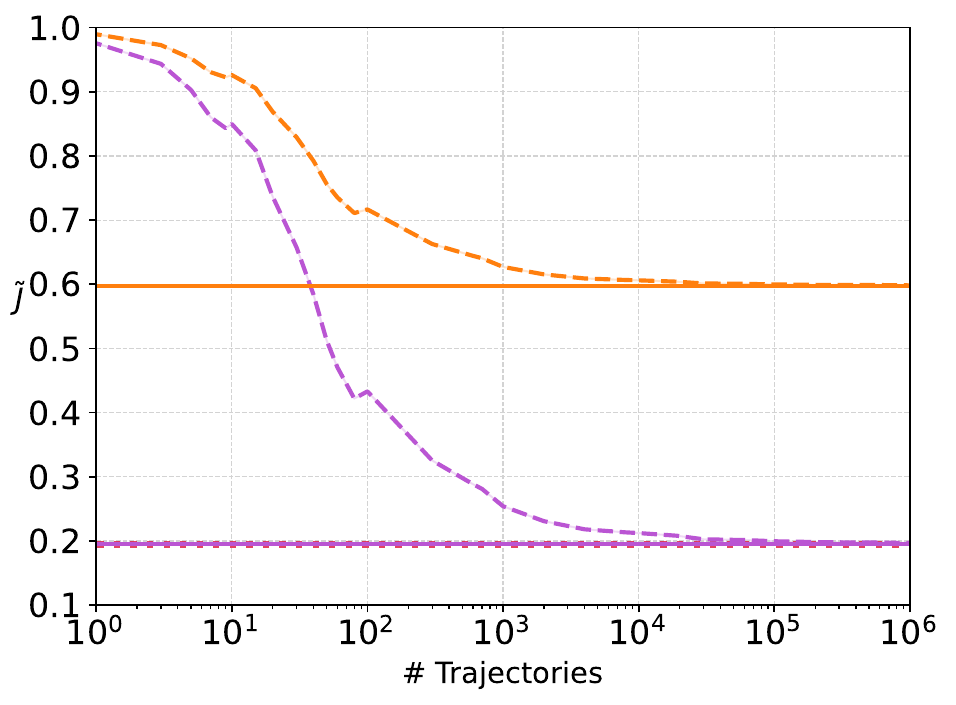}
        \hfill
        \includegraphics[width=0.31\textwidth]{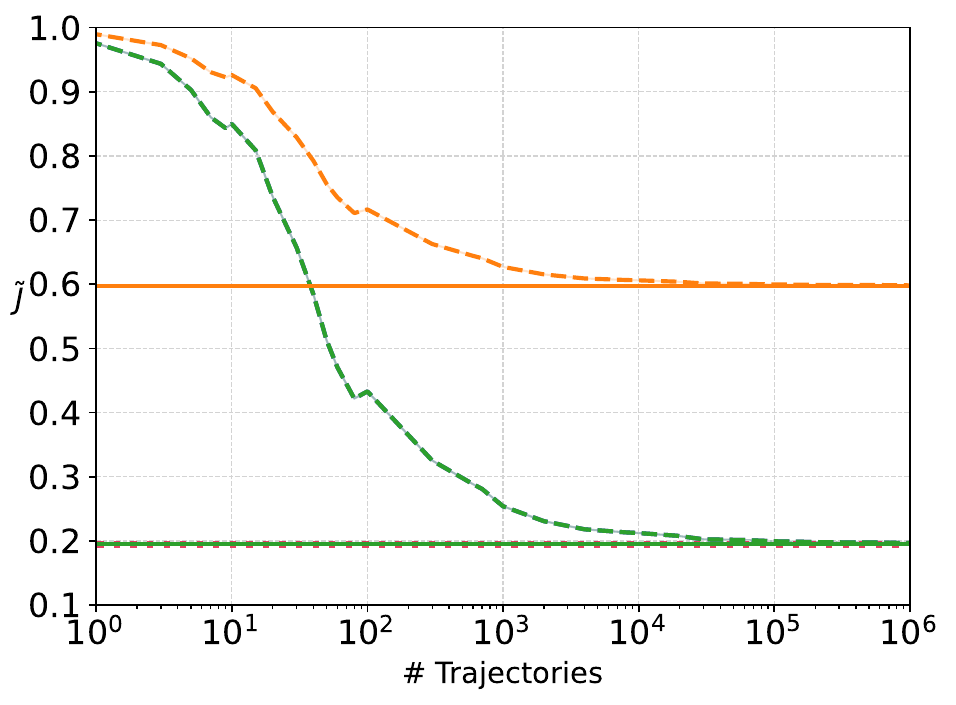}
        \hfill
        % Row 3, second subfigure
        \includegraphics[width=0.31\textwidth]{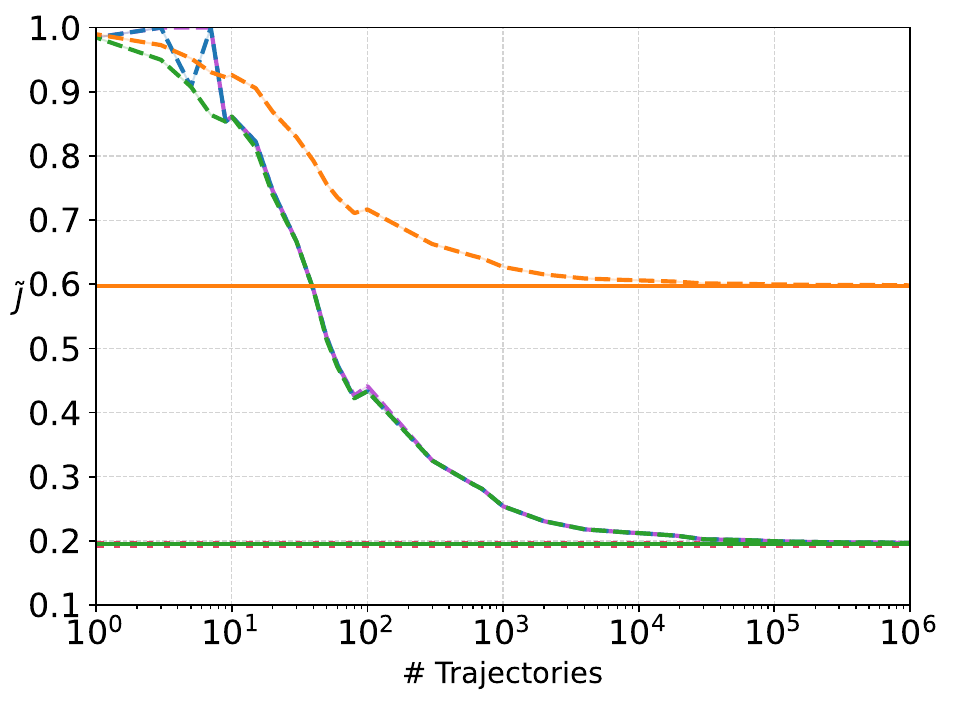}
        \caption{Firewire.}
    \end{subfigure}
\vspace{0.2em}
        \begin{subfigure}[t]{\textwidth}
        \centering
        % Row 3, first subfigure
        \includegraphics[width=0.31\textwidth]{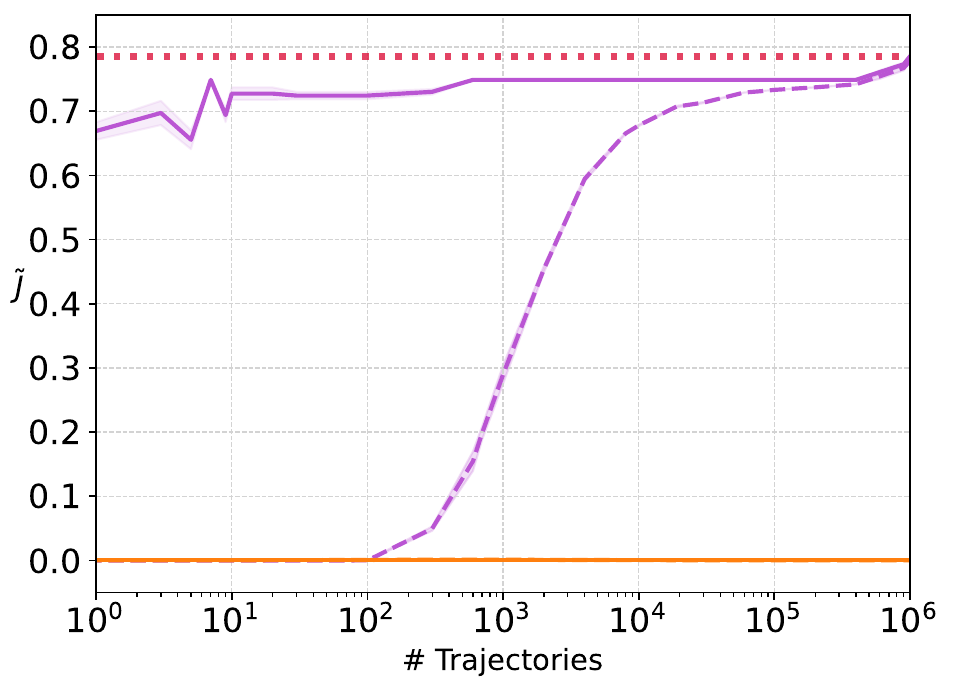}
        \hfill
        \includegraphics[width=0.31\textwidth]{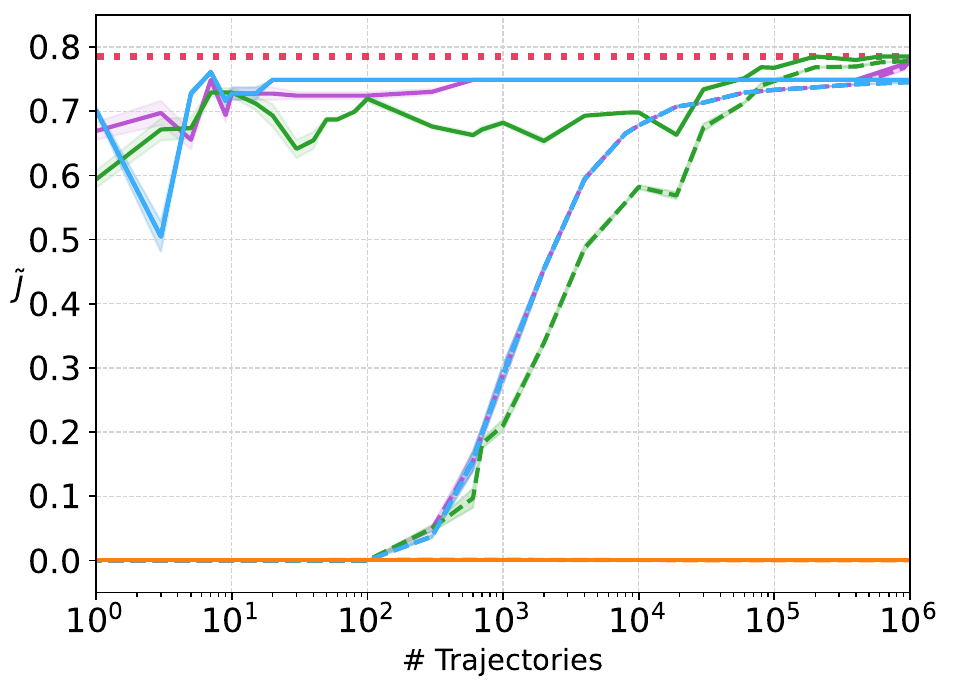}
        \hfill
        % Row 3, second subfigure
        \includegraphics[width=0.31\textwidth]{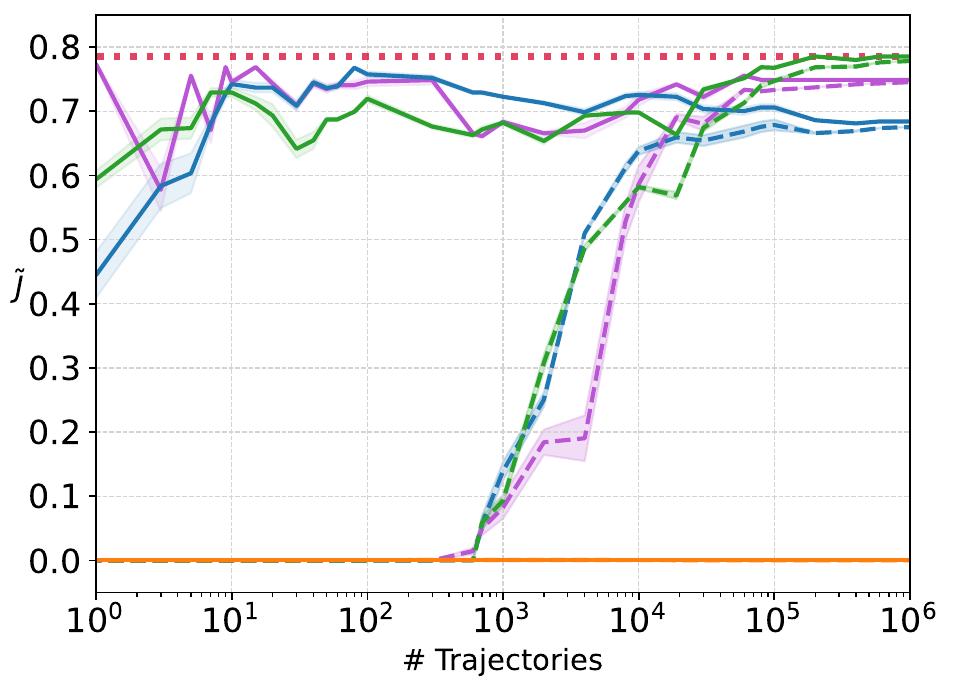}
        \caption{Semi-Autonomous Vehicle.}
    \end{subfigure}

    \vspace{0.2em} % Ad
\end{figure}
\begin{figure}\ContinuedFloat
    \centering
    \begin{subfigure}[t]{\textwidth}
        \centering
        % Row 3, first subfigure
        \includegraphics[width=0.31\textwidth]{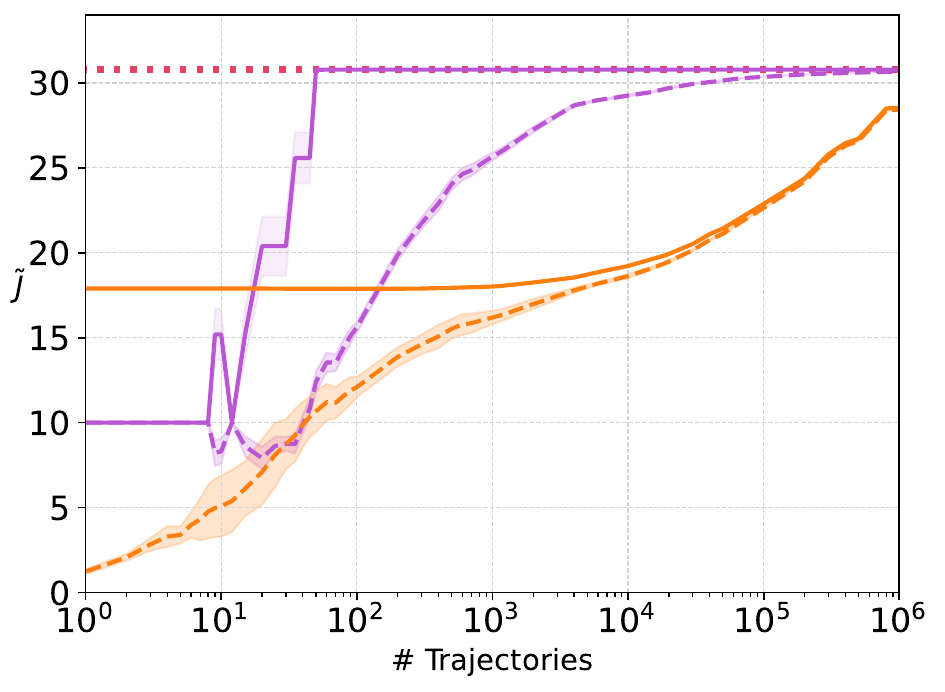}
        \hfill
        \includegraphics[width=0.31\textwidth]{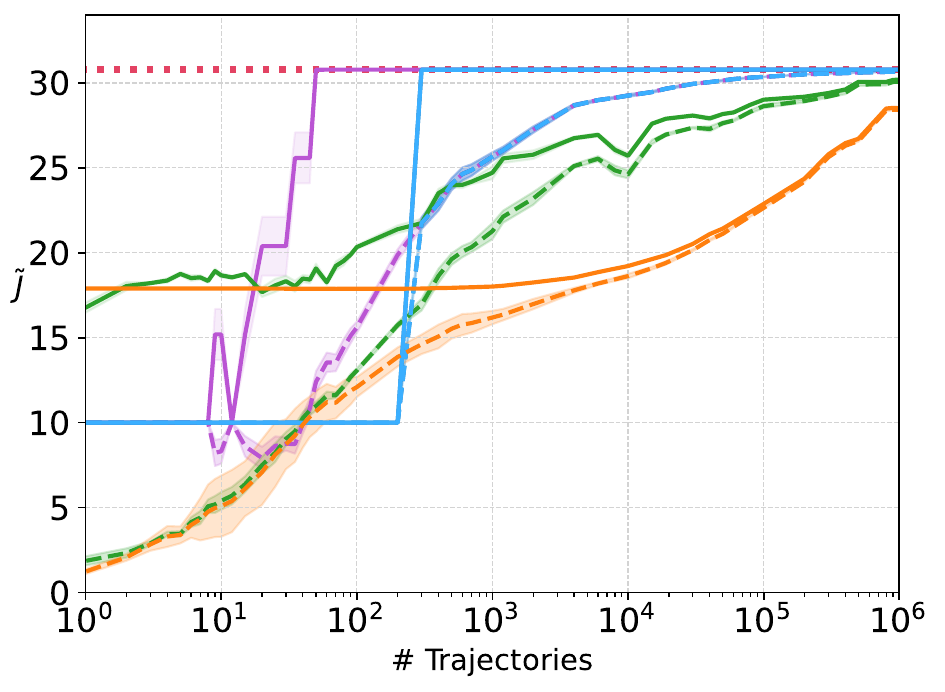}
        \hfill
        % Row 3, second subfigure
        \includegraphics[width=0.31\textwidth]{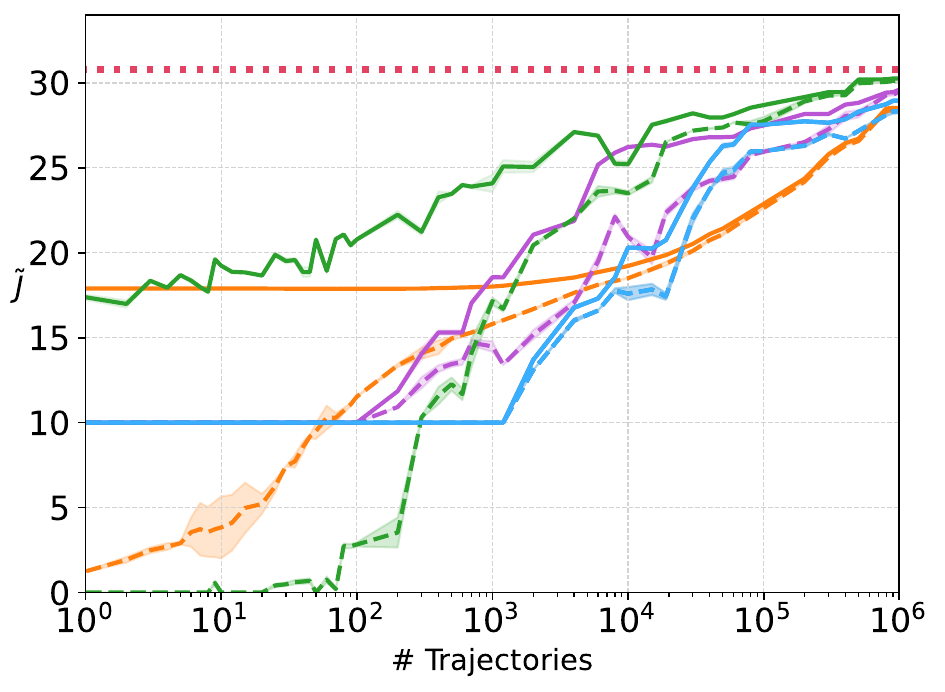}
        \caption{Betting Game.}
    \end{subfigure}

    \begin{subfigure}[t]{\textwidth}
        \centering
        % Row 3, first subfigure
        \includegraphics[width=0.31\textwidth]{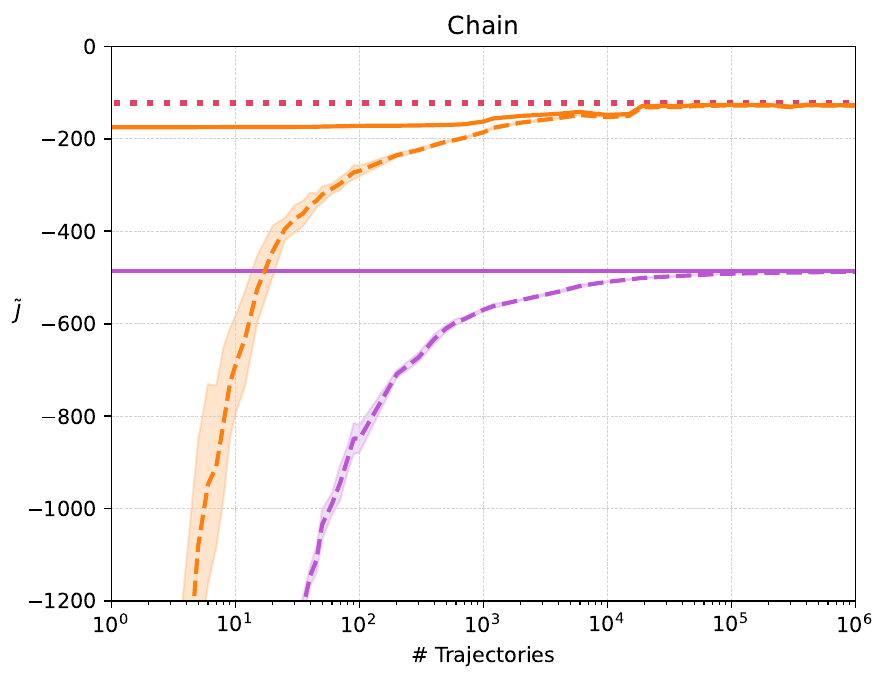}
        \hfill
        \includegraphics[width=0.31\textwidth]{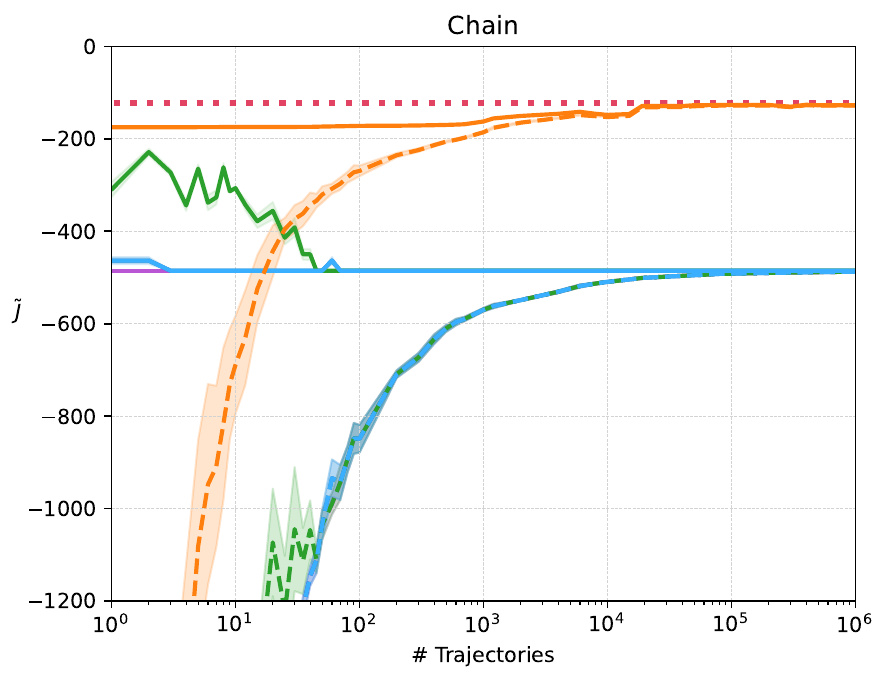}
        \hfill
        % Row 3, second subfigure
        \includegraphics[width=0.31\textwidth]{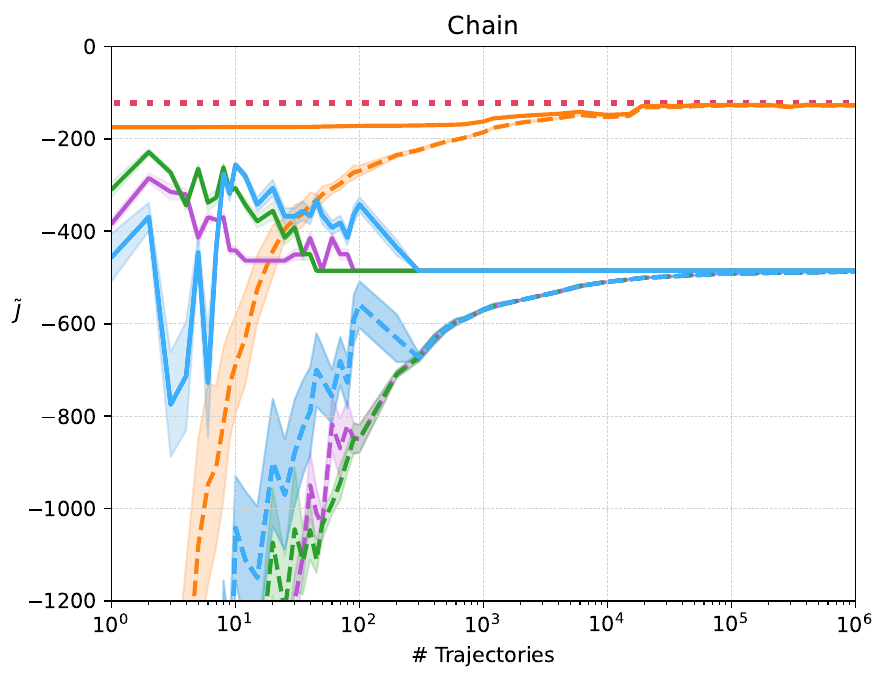}
        \caption{Chain.}
    \end{subfigure}
    
    \caption{Full process of policy training and performance quantification for the best performing IMDP policy and RL policy (left), all IMDP policies with model-based optimisations (middle) and all IMDP policies without optimisations (right).}
     \label{fig:resfull}
     
\end{figure}

\section{IMDP Learning Algorithms}
\label{sec:appimdp}
We detail the learning algorithms used: (1) PAC learning (Section~\ref{sec:paclearning}), (2) Linearly Updating Intervals~\citep{DBLP:conf/nips/SuilenS0022}, (3) UCRL2 reinforcement learning~\citep{DBLP:conf/nips/AuerJO08}, and (4) Maximum a-posteriori (MAP) point estimates~\citep{DBLP:conf/nips/SuilenS0022}. PAC learning is fully described in Section~\ref{sec:paclearning} and is applied in policy learning the exact same way we use it in the learning of IMDP overapproximations of the verification set.

\subsection{Linearly Updating Intervals}
Linearly Updating Intervals (LUI) is a recent approach for learning IMDPs from sample trajectories of an unknown MDP, introduced in~\cite{DBLP:conf/nips/SuilenS0022}. It exploits the Bayesian approach of intervals with linearly updating conjugate priors~\citep{DBLP:conf/nips/SuilenS0022,walter2009imprecision}. Although the learned IMDP does not guarantee inclusion of the underlying MDP, it has been empirically shown to be tighter while remaining sound. For each uncertain transition \(P(s,a,s')\), LUI updates the interval \(P^I = [\underline{P}^I, \overline{P}^I]\), known as the \textit{prior interval}, and the \textit{prior strength} \(n\).

Given state-action count \(N = \#(s,a)\) and transition count \(k = \#(s,a,s')\) from sample trajectories, the prior interval is updated to the \textit{posterior interval}, as follows:
\[
    \underline{P}^I = \frac{n\underline{P}^I + k}{n + N},
\]
\[
    \overline{P}^I = \frac{n\overline{P}^I + k}{n + N},
\]
with the posterior strength \(n' = n + N\). In our experiments, we initialize the prior intervals for each unknown transition as \([\varepsilon, 1]\) and set the prior strength to \(n = 0\).

\subsection{Maximum A-Posteriori Point Estimates}

Maximum a-posteriori (MAP) point estimates are a well-known principle from Bayesian statistics and parameter estimation~\citep{DBLP:books/lib/Bishop07}. Given an unknown MDP with transition probabilities \(P(s,a,s_i)\) for the \(m\) successors \(s_1, \dots, s_m\) of a state-action pair \((s,a)\), the probability of observing \(k_i = \#(s,a,s_i)\) transitions for each successor, given \(N = \#(s,a)\) trials, follows a \textit{multinomial distribution}:

\[
    f(k_1, \dots, k_m \mid P) = \frac{N!}{k_1!\cdot \dots \cdot k_m!} \cdot \prod_{i = 1}^{m}P(s,a,s_i).
\]

Using the Dirichlet distribution as the conjugate prior of the multinomial distribution~\citep{conjugatepriors}, we can obtain a posterior distribution over the unknown parameters \(P(s,a,s_i)\) by updating the prior parameters \(\alpha_1, \dots, \alpha_m\) of the Dirichlet distribution to \(\alpha_1 + k_1, \dots , \alpha_m + k_m\). The MAP point estimate is the mode of the resulting Dirichlet distribution, computed as:

\[
    \tilde{P}(s,a,s_i) = \frac{\alpha_i - 1}{\left(\sum_{j=1}^m \alpha_j\right) - m}.
\]

We employ the MAP point estimate as point intervals \([\tilde{P}(s,a,s_i), \tilde{P}(s,a,s_i)]\). In our experiments, we initialize all Dirichlet priors to be uniform distributions with \(\alpha_i = 1\).

\subsection{UCRL2}
UCRL2 is an established reinforcement learning algorithm introduced in~\cite{DBLP:conf/nips/AuerJO08}, designed to handle the exploration-exploitation trade-off in an unknown environment. We adapt a modified version from~\cite{DBLP:conf/nips/SuilenS0022} for IMDP learning. Similar to PAC learning (see Section~\ref{sec:paclearning}), we build transition probability intervals by expanding the frequentist point estimate:

\[
    \tilde{P}(s,a,s') = \frac{\#(s,a,s')}{\#(s,a)},
\]

for a transition \((s,a,s')\) by \(\delta\):

\[
    P^{\gamma}(s,a,s') = [\max(\mu, \tilde{P}(s,a,s') - \delta), \min(\tilde{P}(s,a,s') + \delta, 1)].
\]

The interval width \(\delta\) for UCRL2 is defined as:

\[
    \delta = \sqrt{\frac{14|S| \cdot \log(2|A| \cdot |T| \cdot 1/\gamma)}{\#(s,a)}},
\]

where \(|S|\) is the number of states, \(|A|\) is the number of actions, and \(|T|\) is the total number of transitions~\citep{DBLP:conf/nips/AuerJO08,DBLP:conf/nips/SuilenS0022}. For unvisited state-action pairs with \(\#(s,a) = 0\), we use the interval \([\mu, 1]\) as in PAC learning.

}

\end{document}